\documentclass{article} % For LaTeX2e
\usepackage{amsmath,amsfonts,bm}

\def\eqref#1{equation~\ref{#1}}
\def\1{\bm{1}}

\DeclareMathAlphabet{\mathsfit}{\encodingdefault}{\sfdefault}{m}{sl}
\SetMathAlphabet{\mathsfit}{bold}{\encodingdefault}{\sfdefault}{bx}{n}

\usepackage[utf8]{inputenc}   % UTF-8 input
\usepackage[T1]{fontenc}      % T1 font encoding
\usepackage{times}            % Times font

\usepackage{amsmath,amssymb,amsfonts,mathtools} % math symbols & tools
\usepackage{amsthm}           % theorem environments
\usepackage{graphicx}         % include graphics
\usepackage{subcaption}       % subfigures
\usepackage{wrapfig}          % wrap text around figures
\usepackage{float}            % improved float placement
\usepackage{booktabs}         % professional tables
\usepackage{multirow}         % table multirow cells
\usepackage{multicol}         % multi-column text
\usepackage{tabularx}         % tables with flexible width
\usepackage{colortbl}         % colored table cells
\usepackage{arydshln}         % dashed lines in tables

\usepackage{microtype}        % better typography
\usepackage{nicefrac}         % compact fractions like ½
\usepackage{url}              % simple URL typesetting
\usepackage[usenames,dvipsnames]{xcolor}       % color support
\usepackage[inline]{enumitem}         % custom enumerate/itemize
\usepackage[pagebackref,breaklinks,colorlinks]{hyperref} % hyperlinks
\usepackage[capitalize,noabbrev]{cleveref} % smart cross-refs

\usepackage{algorithm}        % algorithms (float)
\usepackage{algorithmic}      % algorithmic environment

\usepackage[textsize=tiny]{todonotes} % inline TODO notes
\usepackage{lipsum}           % dummy text
\usepackage{tocloft}          % table of contents formatting
\usepackage{etoc}             % per-section ToCs
\usepackage[toc,page,header]{appendix} % appendix management
\usepackage{etoolbox}         % booleans, toggles
\usepackage{scalerel}         % scalable symbols
\usepackage{placeins}         % \FloatBarrier
\usepackage{dsfont}           % double-struck symbols (e.g. ��)
\usepackage{leftindex}        % left subscripts/superscripts
\usepackage{overpic}          % overlay text on figures
\usepackage{cancel}           % strike-through in math
\usepackage{linguex}          % linguistics examples
\usepackage[most]{tcolorbox}  % colored boxes
\usepackage{mdframed}         % framed environments

\usepackage{tikz}             % drawing
\usetikzlibrary{arrows.meta,bending,positioning,3d} % TikZ extras
\usepackage{tikz-3dplot}      % 3D coordinate plots

\theoremstyle{plain}
\newtheorem{theorem}{Theorem}[section]
\newtheorem{proposition}[theorem]{Proposition}

\newtheorem{corollary}[theorem]{Corollary}

\theoremstyle{definition}

\theoremstyle{remark}

\definecolor{iccvblue}{rgb}{0.21,0.49,0.74}
\definecolor{cvprblue}{rgb}{0.21,0.49,0.74}
\definecolor{mygreen}{RGB}{129,178,154}
\definecolor{myblue}{RGB}{51,92,103}
\definecolor{myyellow}{RGB}{224,159,62}
\definecolor{myred}{RGB}{158,42,43}
\definecolor{mydarkred}{RGB}{84,11,14}
\definecolor{mywhite}{RGB}{255,243,176}
\definecolor{forestGreen}{RGB}{34,139,34}

\hypersetup{ 
    colorlinks,
    linkcolor={red!50!black},
    citecolor={blue!50!black},
    urlcolor={blue!80!black}
}

\newcommand{\RETURN}{\textbf{return}\ }

\makeatletter
\def\adl@drawiv#1#2#3{%
\hskip.5\tabcolsep
\xleaders#3{#2.5\@tempdimb #1{1}#2.5\@tempdimb}%
#2\z@ plus1fil minus1fil\relax
\hskip.5\tabcolsep}
\newcommand{\cdashlinelr}[1]{%
\noalign{\vskip\aboverulesep
\global\let\@dashdrawstore\adl@draw
\global\let\adl@draw\adl@drawiv}
\cdashline{#1}
\noalign{\global\let\adl@draw\@dashdrawstore
\vskip\belowrulesep}}
\makeatother

\usepackage[preprint]{icml2026}
\usepackage{times}
\usepackage{comment}
\date{}

\begin{document}

\twocolumn[
  \icmltitle{Multi-way Representation Alignment}
  % It is OKAY to include author information, even for blind submissions: the
  % style file will automatically remove it for you unless you've provided
  % the [accepted] option to the icml2026 package.

  % List of affiliations: The first argument should be a (short) identifier you
  % will use later to specify author affiliations Academic affiliations
  % should list Department, University, City, Region, Country Industry
  % affiliations should list Company, City, Region, Country

  % You can specify symbols, otherwise they are numbered in order. Ideally, you
  % should not use this facility. Affiliations will be numbered in order of
  % appearance and this is the preferred way.
  \icmlsetsymbol{equal}{*}

  \begin{icmlauthorlist}
    \icmlauthor{Akshit Achara}{kcl}
    \icmlauthor{Tatiana Gaintseva}{qmul}
    \icmlauthor{Matéo Mahaut}{upf}
    \icmlauthor{Pritish Chakraborty}{iitb}
    \icmlauthor{Viktor Stenby Johansson}{dtu}
    \icmlauthor{Melih Barsbey}{icl}
    \icmlauthor{Emanuele Rodol\`a}{sur,par}
    \icmlauthor{Donato Crisostomi}{sur}
  \end{icmlauthorlist}

  \icmlaffiliation{kcl}{King's College London, London, United Kingdom}
  \icmlaffiliation{qmul}{Queen Mary University of London, London, United Kingdom}
  \icmlaffiliation{upf}{Universitat Pompeu Fabra, Barcelona, Spain}
  \icmlaffiliation{iitb}{Indian Institute of Technology, Bombay}
  \icmlaffiliation{dtu}{Technical University of Denmark}
  \icmlaffiliation{icl}{Imperial College London, London, United Kingdom}
  \icmlaffiliation{sur}{Sapienza University of Rome}
  \icmlaffiliation{par}{Paradigma}

  \icmlcorrespondingauthor{Donato Crisostomi}{crisostomi@di.uniroma1.it}
  % \icmlcorrespondingauthor{Firstname2 Lastname2}{first2.last2@www.uk}

  % You may provide any keywords that you find helpful for describing your
  % paper; these are used to populate the "keywords" metadata in the PDF but
  % will not be shown in the document
  \icmlkeywords{Representation Learning, Interpretability, Universal Representations, Vision}

  \vskip 0.3in
]
\printAffiliationsAndNotice{}  % no special notice (required even if empty)

\begin{abstract}
    The Platonic Representation Hypothesis suggests that independently trained neural networks converge to increasingly similar latent spaces. However, current strategies for mapping these representations are inherently pairwise, scaling quadratically with the number of models and failing to yield a consistent global reference. In this paper, we study the alignment of $M \ge 3$ models. We first adapt Generalized Procrustes Analysis (GPA) to construct a shared orthogonal universe that preserves the internal geometry essential for tasks like model stitching. We then show that strict isometric alignment is suboptimal for retrieval, where agreement-maximizing methods like Canonical Correlation Analysis (CCA) typically prevail. To bridge this gap, we finally propose Geometry-Corrected Procrustes Alignment (GCPA), which establishes a robust GPA-based universe followed by a post-hoc correction for directional mismatch. Extensive experiments demonstrate that GCPA consistently improves any-to-any retrieval while retaining a practical shared reference space.
\end{abstract}

\section{Introduction}\label{introduction}
\begin{figure}[htbp]
    \centering
    \begin{tikzpicture}[scale=0.9, >=latex, font=\small]
        % --- First figure (Pairwise) ---
        \begin{scope}[local bounding box=first]
            \def\r{2}
            \def\angleX{90}
            \def\angleY{-30}
            \def\angleZ{210}

            \coordinate (X1) at (\angleX:\r);
            \coordinate (Y1) at (\angleY:\r);
            \coordinate (Z1) at (\angleZ:\r);
            % Pairwise arrows
            \draw[-latex] (X1) to[bend left] node[midway, right] {$\Omega_{Y \leftarrow X}$} (Y1);
            \draw[-latex] (Y1) to[bend left] node[midway, below] {$\Omega_{Z \leftarrow Y}$} (Z1);
            \draw[-latex] (Z1) to[bend left] node[midway, left] {$\Omega_{X \leftarrow Z}$} (X1);
            % Draw points
            \filldraw (X1) circle (1.5pt) node[above] {X};
            \filldraw (Y1) circle (1.5pt) node[below right] {Y};
            \filldraw (Z1) circle (1.5pt) node[below left] {Z};
        \end{scope}
        % --- Second figure (Cycle-Consistent), shifted to the right ---
        \begin{scope}[shift={(4.5,0)}, local bounding box=second]
            \coordinate (U) at (0,0);
            \def\r{2}
            \def\angleX{90}
            \def\angleY{-30}
            \def\angleZ{210}
            \coordinate (X2) at (\angleX:\r);
            \coordinate (Y2) at (\angleY:\r);
            \coordinate (Z2) at (\angleZ:\r);
            %
            % Arrows to U
            \draw[-latex] (X2) to[bend left] node[midway, right] {$\Omega_X^\top$} (U);
            \draw[-latex] (Y2) to[bend left] node[midway, below] {$\Omega_Y^\top$} (U);
            \draw[-latex] (Z2) to[bend left] node[midway, left] {$\Omega_Z^\top$} (U);
            % Arrows from U
            \draw[-latex] (U) to[bend left] node[midway, left] {$\Omega_X$} (X2);
            \draw[-latex] (U) to[bend left] node[midway, above] {$\Omega_Y$} (Y2);
            \draw[-latex] (U) to[bend left] node[midway, right] {$\Omega_Z$} (Z2);
            % Draw points
            \filldraw (X2) circle (1.5pt) node[above] {X};
            \filldraw (Y2) circle (1.5pt) node[below right] {Y};
            \filldraw (Z2) circle (1.5pt) node[below left] {Z};
            \filldraw (U) circle (1.5pt) node[above right] {U};
        \end{scope}
    \end{tikzpicture}
    \caption{Pairwise alignment (left) learns a separate map for each ordered pair, which does not enforce consistency when maps are composed. Universe alignment (right) learns one map per model into a shared reference $U$, enabling translation between models by composition.
}
    \label{fig:cycle-cons-teaser}
\end{figure}
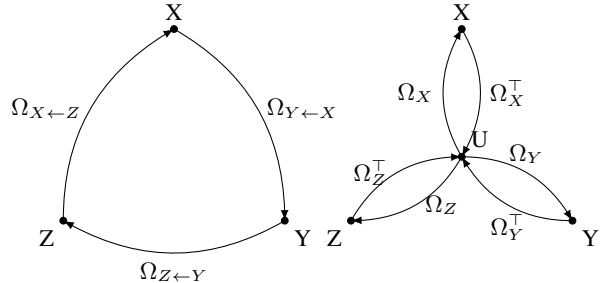

Deep networks succeed largely because of the representation spaces they build. Strikingly, these spaces often resemble each other even when the models differ in architecture, data, optimization, or random seed \citep{kornblith2019similarity, moschellarelative, maiorca2023latent, huh2024platonic}. This observation is captured by the Platonic Representation Hypothesis: independently trained models tend to recover a shared statistical structure of the world in their latent representations \citep{huh2024platonic}.

Motivated by these findings, a growing body of work studies representation alignment as a way to make independently trained models interoperable. Methods range from simple linear or orthogonal mappings \citep{maiorca2023latent} to richer functional correspondences \citep{fumero2024latent}. When alignment succeeds, it unlocks practical use cases including model stitching, cross-modal transfer, and zero-shot composition across models \citep{maiorca2023latent, moschellarelative, norelli2023asif}.

Yet most alignment pipelines are still pairwise (\cref{fig:cycle-cons-teaser}, left): they fit one map for each ordered pair of models and train each map independently. This is poorly matched to settings where many models are used together and translations are obtained by composing maps. In the $M$-model regime, pairwise alignment scales quadratically in the number of maps, requires fitting $(M{-}1)$ new maps to incorporate a new model, and provides no mechanism to ensure that compositions are consistent. As a result, the translation from $X$ to $Z$ can differ depending on whether it is learned directly or obtained by composing through intermediate models, even when each pairwise map performs well under pairwise evaluation. This can make pairwise alignment unreliable as a shared reference for larger collections of models.

A natural alternative is to factor translation through a shared universe space (\cref{fig:cycle-cons-teaser}), learning one map per model to and from a common reference. In the orthogonal case this yields a generalized Procrustes formulation. Translation between arbitrary model pairs is obtained via the universe, which reduces the number of learned maps to $O(M)$ and provides a shared coordinate system that can be reused across tasks. Importantly, this construction encourages consistency under composition: the translation between any two models is unique and independent of the path taken through intermediate models.
This supports multi-model use, simplifies model addition, and enables shared-coordinate analyses such as probing and aggregation.

% In this work, we study alignment among three or more models using a shared universe as a common reference. We first adapt Generalized Procrustes Analysis (GPA)~\cite{schonemann1966generalized,gower1975generalized} and Generalized Canonical Correlation Analysis (GCCA)~\cite{horst1961generalized,kettenring1971canonical} to representation alignment and show that they target complementary objectives. GPA yields an orthogonal universe that preserves each model’s internal structure, which supports tasks that require geometric fidelity such as probing and stitching. GCCA instead emphasizes cross-model agreement and often favors similarity-based retrieval, where matched items should be directly comparable under a shared similarity measure. To bridge these regimes, we propose Geometry-Corrected Procrustes Alignment (GCPA), which first constructs an orthogonal universe with GPA and then applies a small shared correction in universe coordinates that reduces residual directional mismatch while limiting departures from the GPA geometry.

% We evaluate these methods on benchmarks spanning both probing and stitching, as well as multilingual, multimodal and cross-camera retrieval, additionally studying the multi-model regime directly by using three or more models jointly rather than only comparing pairs. The shared-universe construction preserves strong pairwise performance after mapping while improving multi-model evaluations, and the corrected universe yields consistent gains without waiving a practical shared reference space.

In this work, we study the alignment of three or more models through a shared universe reference. We first adapt Generalized Procrustes Analysis (GPA)~\citep{schonemann1966generalized,gower1975generalized} to the representation alignment setting, demonstrating that its orthogonal nature yields a universe that preserves each model's internal structure, a property essential for tasks requiring geometric fidelity such as model stitching. We then empirically show that geometry preservation alone is insufficient for tasks like zero-shot retrieval, as methods maximizing cross-model agreement such as Generalized Canonical Correlation Analysis (GCCA)~\cite{horst1961generalized,kettenring1971canonical} outperform strict isometries. To bridge this gap, we propose Geometry-Corrected Procrustes Alignment (GCPA), which combines the best of both worlds. GCPA first establishes a robust orthogonal universe via GPA, then applies a post-hoc shared correction to minimize residual directional mismatch.

We evaluate these approaches on benchmarks spanning inter-architecture probing, multilingual translation, and cross-camera retrieval. Our results show that while naturally multi-way techniques are necessary for scaling, the specific choice of objective matters: while GPA provides a consistent geometric reference, the proposed geometry-corrected universe (GCPA) is required to achieve state-of-the-art performance in any-to-any retrieval while retaining the practical benefits of the shared Procrustes frame.

\paragraph{Contributions.}
\begin{enumerate}
    \item We extend orthogonal representation alignment to three or more models by adapting Generalized Procrustes Analysis (GPA) to the neural representation setting.
    \item We demonstrate that while GPA provides a consistent global reference superior to pairwise maps, it lags behind agreement-maximizing methods (like GCCA) on retrieval tasks.
    \item We introduce Geometry-Corrected Procrustes Alignment (GCPA) to bridge this gap, establishing a robust universe via GPA and applying a post-hoc correction for directional mismatch to achieve superior performance across diverse datasets and settings.
\end{enumerate}

\section{Related Work}\label{relatedwork}
\paragraph{Representational similarity}
Work on representational similarity typically (i) measures how closely different learned spaces correspond
\citep{kornblith2019similarity, klabunde2025similarity, glielmoRanking, huh2024platonic, acevedo2025approachidentifysemanticallyinformative},
(ii) leverages such similarity to align spaces via explicit maps
\citep{moschellarelative, maiorca2023latent, fumero2024latent, cannistracibricks},
or (iii) studies similarity and alignment jointly \citep{fumero2024latent}.
The applications span multi-modal alignment~\citep{norelli2023asif, Cicchetti2025-ft, yue2025escaping, groger2025limited}, cross-view alignment~\citep{huang2025c3po}, as well as cross-lingual alignment \citep{jawanpuria2019learning}.
While there are methods not requiring paired data~\cite{pmlr-v89-grave19a, jha2026harnessinguniversalgeometryembeddings, schnaus2025s}, we study in this paper the common case where the correspondence is available.
A common and effective alignment tool in this line is Procrustes analysis, which fits an orthogonal map between two spaces and preserves their internal distances and angles
\citep{hurley1962procrustes, pmlr-v89-grave19a, maiorca2023latent}.
More broadly, our setting can be viewed as a special case of multi-space representation learning focused on alignment~\citep{li2018survey}, and is connected to backward-compatible representation learning, which aims to keep updated embeddings comparable to earlier ones \citep{shen2020towards, zhang2022towards}, as well as to communication between pretrained networks \citep{mahaut2025referential, rameshcommunicating}.

\paragraph{Multiset alignment}
An alternative to pairwise mapping is to learn a shared latent space from multiple views via multiset canonical correlation analysis (GCCA) and related variants \citep{horst1961generalized, kettenring1971canonical, fu2017scalable, sorensen2021generalized}. These methods provide a principled $M$-way objective and are natural baselines for multi-view alignment, especially in retrieval-style settings where matched items should be directly comparable in a common set of coordinates. In contrast, our focus is on a single shared reference learned with one map per model, so that translations between models are obtained by composing through the reference and new models can be incorporated by fitting only one additional map.

\section{Methodology}\label{methodology}
We consider the problem of aligning $M \ge 3$ distinct representation spaces $\{X_m\}_{m=1}^M$, where $X_m \in \mathbb{R}^{N \times d}$ contains $N$ matched samples (or is padded to a common dimension). In our setup, all representations are standardized within each space.  Our goal is to establish a shared coordinate system that enables any-to-any translation.

\subsection{The Universe Factorization}
\label{sec:universe_factorization}

Standard alignment approaches are inherently pairwise. To map space $n$ to space $m$, one typically minimizes the discrepancy between matched pairs:
\begin{equation}
\mathcal{L}_{\text{pair}}(\Omega) = \sum_{m<n} \| X_m - X_n \Omega_{m \leftarrow n} \|_F^2.
\end{equation}
While effective for two models, this approach scales quadratically, requiring $O(M^2)$ maps for $M$ models. Furthermore, it lacks a global reference; the composition of learned maps often violates cycle consistency (i.e., $\Omega_{m \leftarrow k} \Omega_{k \leftarrow n} \neq \Omega_{m \leftarrow n}$), making multi-hop translation unreliable.

To address this, we adopt an \textbf{object-to-universe factorization}. Rather than learning pairwise maps directly, we assume the existence of a shared statistical ``universe'' $U$ and learn a single map $\Omega_m$ from each model into this reference:
\begin{equation}
\Omega_{m \leftarrow n} := \Omega_m \Omega_n^\top.
\end{equation}
This reduces the complexity from $O(M^2)$ to $O(M)$ and guarantees cycle consistency by design. The central question then becomes: \textit{What are the geometric properties of this universe, and how should $U$ be constructed?}

\subsection{Constructing an Isometric Universe (GPA)}
\label{sec:gpa}

For tasks such as model stitching or latent probing, it is critical that alignment does not distort the internal topology of the individual models. We therefore first consider the construction of an orthogonal universe, where each map $\Omega_m$ is constrained to the orthogonal group $O(d)$.

This formulation aligns with \textbf{Generalized Procrustes Analysis (GPA)}. We seek a consensus centroid $U$ and a set of rotations $\{\Omega_m\}$ that minimize the global dispersion:
\begin{equation}
\min_{\{\Omega_m \in O(d)\},\, U} \sum_{m=1}^M \|X_m \Omega_m - U\|_F^2
\label{eq:gpa_objective}
\end{equation}
Optimization proceeds via alternating minimization:
\begin{enumerate}
    \item \textbf{Update Consensus:} $U \leftarrow \frac{1}{M} \sum_{m=1}^M X_m \Omega_m$.
    \item \textbf{Update Maps:} Solve the orthogonal Procrustes problem for each $\Omega_m$ to align $X_m$ to the current $U$.
\end{enumerate}

Because $\Omega_m$ is an isometry, distances and angles within each model are preserved exactly in the universe ($\| \Omega_m x - \Omega_m y \| = \| x - y \|$). This makes GPA the natural choice for geometric interoperability.

\subsection{The Limits of Isometry in Retrieval}
\label{sec:retrieval_gap}

While GPA provides a robust geometric reference, empirical evidence suggests that strict isometry is suboptimal for \textbf{retrieval tasks}. In cross-modal or zero-shot retrieval, maximizing the cosine agreement between matched samples often requires reshaping the feature spaces to suppress modality-specific noise, a degree of freedom that orthogonal maps lack.

To quantify this ``retrieval gap,'' we contrast GPA with \textbf{Generalized Canonical Correlation Analysis (GCCA)}. GCCA relaxes the orthogonality constraint, allowing linear projections $\Phi_m$ that deform the spaces to maximize shared correlation:
\begin{equation}
\min_{\{\Phi_m\}} \sum_{i < j} \|X_i \Phi_i - X_j \Phi_j\|_F^2 \quad \text{s.t.} \quad \sum_{m=1}^M \Phi_m^\top \Phi_m = I.
\label{eq:gcca_objective}
\end{equation}

% \begin{proposition}[Optimality for Squared Discrepancy]
% The solution to the GCCA objective is given by the eigenvectors of the inter-model covariance matrix. Among all linear shared-basis embeddings, this solution globally minimizes the total pairwise mismatch energy.
% \end{proposition}

The following proposition highlights why GCCA often outperforms isometries in pure retrieval:

\begin{proposition}[GCCA minimizes squared discrepancy]
The global minimizer of Eq.~\ref{eq:gcca_objective} is obtained via a spectral decomposition of a matrix constructed from cross-view correlations (Theorem~\ref{thm:gcca_solution}). Among linear shared-basis embeddings satisfying the constraint, this solution globally minimizes the total pairwise mismatch energy.
\end{proposition}

These considerations pinpoint a trade-off: GCCA achieves minimal mismatch (ideal for retrieval) but does so by distorting the internal geometry of the models (harmful for stitching/probing); GPA preserves geometry but suffers from residual directional mismatch.

\subsection{Geometry-Corrected Procrustes Alignment (GCPA)}
\label{sec:gcpa}

To bridge the gap between geometric fidelity and retrieval performance, we propose \textbf{Geometry-Corrected Procrustes Alignment (GCPA)}.

Our insight is that we can retain the robust orthogonal universe of GPA as a ``scaffold,'' and then apply a shared, non-linear correction to ``polish'' the alignment. This correction minimizes the residual directional mismatch that GPA leaves behind, without discarding the universe coordinate system.

\paragraph{The Consensus Principle.} We observe that in the GPA universe, the average direction of a matched sample across all $M$ models serves as a robust target. Let $\hat{u}_{m,i}$ be the unit-normalized vector for sample $i$ mapped into the universe by model $m$. We define the \textbf{consensus direction} $c_i$:
\begin{equation}
c_i = \operatorname{norm}\left( \frac{1}{M} \sum_{m=1}^M \hat{u}_{m,i} \right).
\label{eq:consensus}
\end{equation}
Moving individual representations toward this consensus strictly improves global agreement. The following identity links consensus agreement to total pairwise agreement:

\begin{proposition}\label{prop:consensus}
Fix a sample index $i$ and let $\{\hat{u}_{m,i}\}_{m=1}^M \subset \mathbb{S}^{d-1}$ denote the
unit-normalized universe directions for that sample across spaces. 

% Define the consensus direction
% \[
% c_i = \mathrm{norm}\!\left(\frac{1}{M}\sum_{m=1}^M \hat{u}_{m,i}\right).
% \]

Then
\begin{align}
    \frac{1}{M}\sum_{m=1}^M \langle \hat{u}_{m,i}, c_i \rangle
    &= \frac{1}{M}\left\|\sum_{m=1}^M \hat{u}_{m,i}\right\|_2, \label{eq:mean_to_consensus}\\
    \sum_{m<n}\langle \hat{u}_{m,i}, \hat{u}_{n,i}\rangle
    &= \frac{1}{2}\left(\left\|\sum_{m=1}^M \hat{u}_{m,i}\right\|_2^2 - M\right).
    \label{eq:pairwise_sum_identity}
\end{align}
\end{proposition}
\emph{i.e.} increasing the cosine similarity of each view to the consensus $c_i$ monotonically increases the sum of pairwise cosine similarities between all views.

\paragraph{The GCPA Algorithm.} GCPA operates in two stages: i) \textbf{GPA Initialization:} We first solve the standard GPA problem to obtain the orthogonal maps $\{\Omega_m\}$ and the base universe representations $X_m^\star = X_m \Omega_m$; ii) \textbf{Geometry Correction:} We train a lightweight, shared residual map $T_\theta$ (a small MLP) that acts on the universe coordinates. The objective minimizes the distance to the consensus while strictly penalizing deviation from the trusted GPA geometry:
\begin{equation}
\mathcal{L}_{\text{correct}} = - \sum_{m,i} \langle T_\theta(\hat{u}_{m,i}), c_i \rangle + \lambda \mathcal{L}_{\text{trust}},
\end{equation}
where $\mathcal{L}_{\text{trust}}$ penalizes $T_\theta$ if the angular deviation from the original GPA vector exceeds a tight threshold. By construction, $T_\theta$ returns a row-normalized direction and $c_i$ is the per-sample consensus direction. 

% Here, $\operatorname{norm}(\cdot)$ denotes L2 normalization (row-wise unit vectors), and $c_i$ is the per-sample consensus direction.

By sharing $T_\theta$ across all models, GCPA refines the universe itself rather than individual models. This allows us to achieve the high agreement typically associated with CCA-style methods while remaining anchored to the geometrically consistent Procrustes frame.

Additional theoretical details (including related theorems and corollaries) are deferred to Appendix~\ref{app:proofs}. Algorithmic summaries for GCPA and GCCA are provided in Appendix~\ref{app:gcpa_details} and Appendix~\ref{app:gcca_algo}. We provide an additional spectral analysis of the corrected universe in Appendix~\ref{app:structural_analysis}, showing how GCPA changes the principal structure of the aligned representations.
\section{Experiments}\label{experiments}
We study multi-model alignment through the lens of a shared universe reference. We begin with probing and stitching tasks, where preserving each model’s internal structure is important. We then consider aggregation in a common coordinate system through clustering. Finally, we study the task of retrieval, first in settings with strong shared structure where shared-basis methods are effective, and then in more heterogeneous settings where an orthogonal universe remains valuable but residual directional mismatch benefits from a universe-level correction.

\subsection{Baselines and evaluation}
\label{sec:baselines-eval}
We compare \textbf{No Alignment (NA)}, \textbf{Pairwise (PW)} orthogonal Procrustes, \textbf{$M$-way GPA}, \textbf{GCCA} (retrieval settings), and \textbf{GCPA}; see \cref{methodology} for definitions.

All alignment parameters are fit on the training split using sample identity to define correspondences across spaces, and results are reported on held-out test sets. For probing, we report stitching accuracy by training a linear probe on one space and evaluating it on another after mapping. For retrieval, we average performance over all ordered pairs $(i\!\rightarrow\!j)$ and include pairwise breakdowns where relevant. For retrieval experiments, we apply PCA on features to a common dimension $d$ prior to alignment (only when model features have different dimensions), fitting PCA on the training split and applying it to validation/test. For probing and clustering, we use the original feature spaces without PCA.

We use a fixed, dataset-agnostic GCPA configuration for all main results. Appendix~\ref{app:hyperparams} shows that additional gains are possible with an alternative set of parameters. We also discuss the computational costs of the methods in Appendix~\ref{app:computational-complexity}.

\subsection{Geometric interoperability}
\label{sec:gpa_exps}

\subsubsection{Multi-way Alignment Strengthens Weak Pairwise Links}
\label{sec:weak-links}
We investigate the ability of a shared universe to stabilize the alignment between two models that share little surface-level similarity. To test this, we isolate a fragile pair of models with poor direct correspondence induced by training them on disparate views of the data. We then attempt to align them via a universe of healthy anchor models. As a dataset, we utilize a version of CIFAR-100 \citep{krizhevsky2009cifar100} where we introduce a distribution shift by replacing 85\% of the images with binary edge maps, as described in further detail in Appendix~\ref{app:weak_links}.

We first establish a baseline by fitting a direct pairwise orthogonal map between the fragile pair. Due to the distribution shift between their training representations, this direct link performs poorly. We then construct an $M$-way GPA universe containing the fragile pair alongside a set of standard models. We progressively expand this set by adding more anchor models to the universe and re-evaluate the translation quality between the original fragile pair when routed through the shared reference.

\begin{figure}[t]
    \centering
    \includegraphics[width=\linewidth]{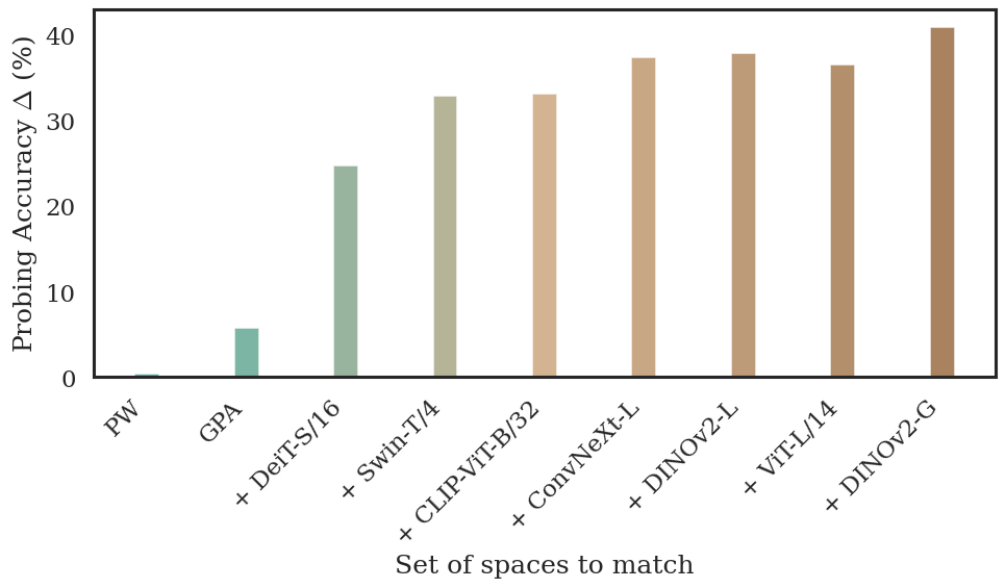}
    \caption{\textbf{Multi-way alignment stabilizes fragile connections.} On edge-heavy CIFAR-100, we isolate a ``weak'' model pair with poor alignment. By progressively expanding the universe with robust models and refitting the universe, we observe a monotonic increase in stitching accuracy between the original fragile pair.}\label{fig:expanding-matching-set}
\end{figure}

As shown in \cref{fig:expanding-matching-set}, the universe acts as a geometric hub. While the direct pairwise map fails to capture the correspondence, increasing the number of diverse models in the universe stabilizes the alignment. The shared reference leverages the structural consensus of the anchor models to triangulate the relationship between the fragile pair, effectively healing the weak link. (See \cref{app:weak_links} for experimental details on the expanding set). 
We also replicate this experiment for GCPA in \cref{app:weak_links}, showing a similar trend.

\subsubsection{Efficiently Extending the Universe}
\label{sec:add-model}
Once a universe is learned, it can be reused rather than rebuilt. Pairwise alignment requires learning $M{-}1$ new maps when extending an aligned set of $M$ models, whereas the shared-universe formulation integrates a new model $X_{M+1}$ by fitting a single orthogonal map $\Omega_{M+1}$ into the existing universe. Translation to and from any existing model then follows by composition:
\begin{equation}
\begin{aligned}
\Omega_{m\leftarrow (M+1)}=\Omega_m \Omega_{M+1}^\top,\\
\Omega_{(M+1)\leftarrow m}=\Omega_{M+1}\Omega_m^\top .
\label{eq:add-model-compose}
\end{aligned}
\end{equation}

On CIFAR-100 probing, we learn an $M$-way universe on a base set of models and then introduce a new model $X_{M+1}$.
We report the average cross-model test accuracy over all directed transfers between the new model and the base models.
We compare \textbf{PW}, which fits the incident pairwise maps; \textbf{GPA-REFIT}, which refits the universe on $M{+}1$ models; and \textbf{GPA-ADD}, which fits only $\Omega_{M+1}$ while keeping the base universe fixed (Appendix~\ref{app:add-model}).

\begin{figure}[htpb!]
    \centering
    \includegraphics[width=\linewidth]{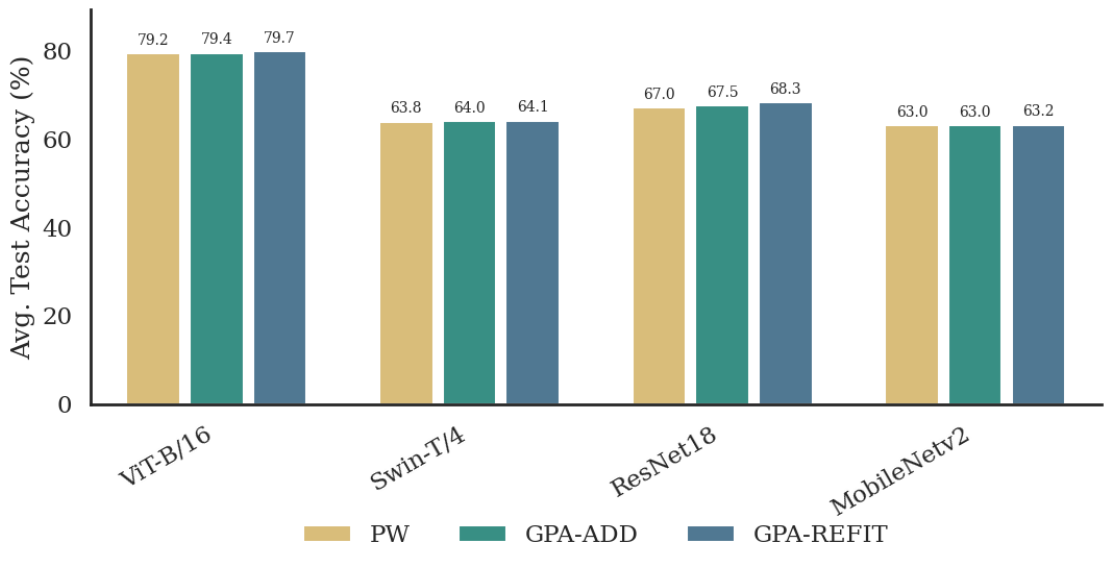}
    \caption{\textbf{Cross-model probing on CIFAR-100.}
    Adding a new model by fitting only $\Omega_{M+1}$ into a fixed universe (GPA-ADD) approaches refitting the universe (GPA-REFIT) and outperforms PW alignment. To cover diverse scenarios, we use four different base model sets where the first two (from the left) sets consist of three models and the next two consist of five models each.}
    \label{fig:addition-bars}
\end{figure}

The results in \cref{fig:addition-bars} demonstrate the efficacy of this approach. Across diverse architectures, including ViT-B/16, Swin-T/4, ResNet18, and MobileNetv2, simply fitting the new model into the fixed universe (\textbf{GPA-ADD}) consistently matches or outperforms the computationally expensive baselines. Specifically, fully refitting the universe (\textbf{GPA-REFIT}) achieves higher average stitching accuracy than independent pairwise alignment (\textbf{PW}) and \textbf{GPA-ADD} as expected and GPA-ADD slightly improves over pairwise alignment. This validates that the shared coordinate system remains robust and can be extended efficiently without the need for global retraining.

\subsection{Aggregation and clustering}
\label{sec:universal-clustering}

Clustering tests a different property than probing or retrieval. Rather than asking whether two models are interchangeable after alignment, it asks whether mapping multiple models into a shared coordinate system yields a space where semantic class structure is easy to recover without supervision. This directly evaluates if multiple pretrained spaces can be aligned into a shared coordinate system consistently across encoders, which can be verified if examples with the same label form tighter groups after alignment.

We evaluate intent clustering on \textsc{MASSIVE}~\cite{fitzgerald2022massive}, which consists of short user queries annotated with one of $60$ intent classes. We embed the same evaluation split with $M{=}10$ pretrained multilingual models (one space per model), so each example provides a matched row across all spaces. We fit alignment on the training split using example identity as correspondence, then map each model's test split embeddings into a shared space and run $k$-means clustering with $k$ set as the number of intents in the test split (for 5 different seeds).

We compare clustering in the original encoder spaces (\textbf{NA}) against clustering after mapping with \textbf{GPA}, \textbf{GCCA}, and our geometry-corrected universe (\textbf{GCPA}). Clustering quality is measured by Adjusted Rand Index (ARI) and Normalized Mutual Information (NMI). We report mean$\pm$std across the 10 models to reflect how consistently each alignment method supports aggregation across different pretrained spaces.

\begin{table}[t]
\centering
\small
\setlength{\tabcolsep}{6pt}
\begin{tabular}{lcc}
\toprule
\textbf{Method} & \textbf{ARI} & \textbf{NMI} \\
\midrule
NA & 0.198$\pm$0.035 & 0.509$\pm$0.041 \\
GPA      & 0.198$\pm$0.036 & 0.509$\pm$0.041 \\
GCCA     & 0.194$\pm$0.032 & 0.508$\pm$0.038 \\
GCPA (ours) & \textbf{0.300$\pm$0.024} & \textbf{0.617$\pm$0.022} \\
\bottomrule \\
\end{tabular}
\caption{\textbf{Clustering on \textsc{MASSIVE}.}
We run k-means separately for each mapped model and report mean$\pm$std across models. GCPA outperforms the other approaches on all 10 models.}
\label{tab:massive_clustering}
\end{table}

Table~\ref{tab:massive_clustering} presents the quantitative results. We observe that standard alignment methods like GPA and GCCA do not improve clustering quality over the original unaligned spaces (NA), suggesting that global rotation or linear projection alone does not enhance the separability of intent classes. In contrast, GCPA achieves a significant improvement in both metrics, boosting the Adjusted Rand Index by over $50\%$ relative to the baselines. This indicates that the post-hoc geometric correction successfully pulls semantically similar items together, creating a shared representation where class structure is more distinct and easier to recover.

\subsection{Retrieval}
\label{sec:gcca_exps}
\begin{figure*}[t]
    \centering
    \includegraphics[width=\linewidth]{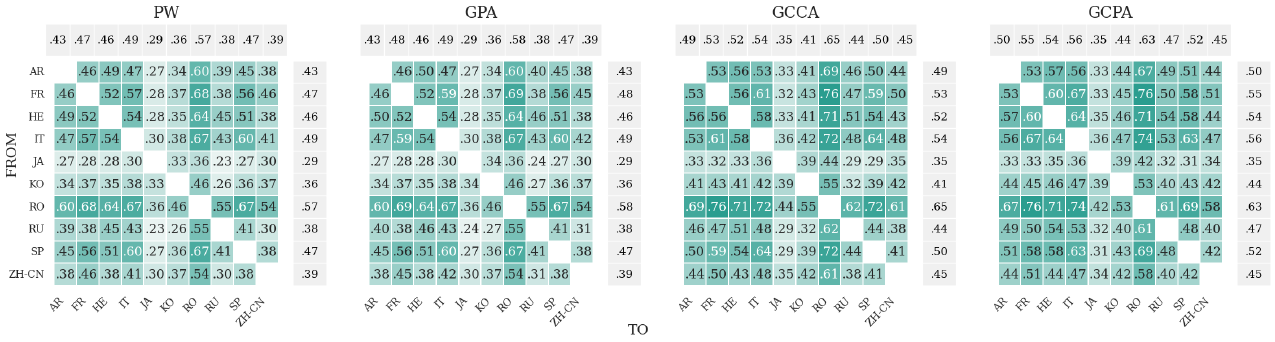}
    \caption{\textbf{Cross-lingual retrieval on \textsc{TED-Multi} (rank-1).} GCPA outperforms GCCA, GPA, and pairwise orthogonal alignment.}
    \label{fig:tedmulti}
\end{figure*}
\subsubsection{Multilingual translation}
\label{sec:tedmulti}
We evaluate cross-lingual sentence retrieval on \textsc{TED-Multi}\footnote{\url{https://huggingface.co/datasets/neulab/ted_multi}}, a multilingual corpus of TED talk transcripts with aligned sentence translations across ten languages.
Each language is embedded with a dedicated pretrained model (Appendix~\ref{app:multilingual-encoders}, Table~\ref{tab:multilingual-encoders}), producing matched-dimensional representations.
Given a query sentence in language $A$, the task is to retrieve its paired translation in language $B$ using cosine similarity; we measure performance as rank-1 retrieval accuracy over the evaluated A$\to$B language pairs.

Direct comparison in the original representation spaces yields near-zero accuracy, indicating that independently trained multilingual models are not natively comparable.
Pairwise orthogonal alignment (PW) substantially improves retrieval, and $M$-way GPA improves further by aligning all languages into a shared universe reference.
Figure~\ref{fig:tedmulti} shows that Geometry-Corrected Procrustes Alignment (GCPA) achieves the best retrieval performance, outperforming both GCCA and purely orthogonal universes (GPA).

Table~\ref{tab:tedmulti_num_spaces} examines how alignment quality scales with the number of languages ($M \in (3, 5, 10)$). For $M{=}3$, we align (EN, FR, ES) and (EN, AR, ZH); for $M{=}5$ we align (EN, FR, ES, IT, RU) and (EN, AR, RU, JA, ZH). While retrieval difficulty naturally increases as more spaces are added, GCPA consistently yields the highest performance across all settings. Crucially, it improves both the average accuracy and the worst-case pair performance relative to GCCA and GPA, demonstrating that the corrected universe remains robust even as the collection of models grows.

\begin{table}[t]
\centering
\small
\setlength{\tabcolsep}{5pt}
\renewcommand{\arraystretch}{1.05}
\resizebox{.5\textwidth}{!}{
\begin{tabular}{lcc c cc c cc}
\toprule
& \multicolumn{2}{c}{$M=3$} & & \multicolumn{2}{c}{$M=5$} & & \multicolumn{2}{c}{$M=10$} \\
\cmidrule(lr){2-3}\cmidrule(lr){5-6}\cmidrule(lr){8-9}
Method & Avg $\uparrow$ & Worst $\uparrow$ & & Avg $\uparrow$ & Worst $\uparrow$ & & Avg $\uparrow$ & Worst $\uparrow$ \\
\midrule
PW   & 0.571 & 0.468 & & 0.471 & 0.303 & & 0.430 & 0.230 \\
GPA  & 0.572 & 0.469 & & 0.474 & 0.308  & & 0.433 & 0.236 \\
GCCA & 0.628 & 0.509 & & 0.531 & 0.366  & & 0.487 & 0.288 \\
GCPA (ours) & \textbf{0.637} & \textbf{0.532} & & \textbf{0.553} & \textbf{0.419} & & \textbf{0.503} & \textbf{0.314} \\
\bottomrule \\
\end{tabular}}
\caption{\textbf{Scaling with the number of languages on \textsc{TED-Multi}.}
Cross-lingual rank-1 retrieval for fixed language subsets at $M{=}3$, $M{=}5$, and $M{=}10$.
We report the average over the evaluated A$\to$B language pairs and the worst-case (minimum) pair performance.}
\label{tab:tedmulti_num_spaces}
\end{table}

\subsubsection{Robustness to Correspondence Noise}
\label{subsec:corrupt-correspondence}
\begin{figure}[t]
    \centering
    \includegraphics[width=0.85\linewidth]{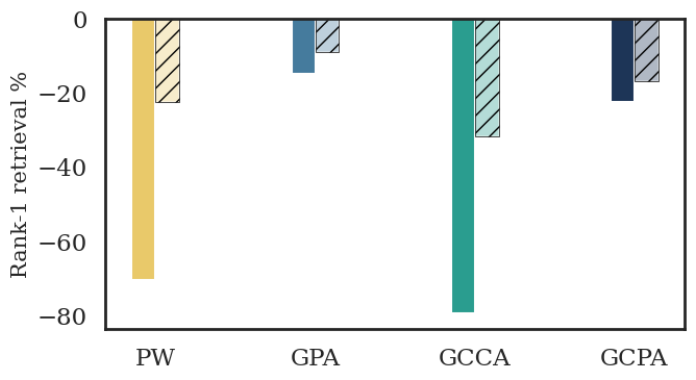}
    \caption{\textbf{Robustness to correspondence noise on \textsc{TED-Multi}.}
Rank-1 retrieval accuracy (\%) on the clean test split relative to the unshuffled baseline. Solid bars average over the evaluated pairs within the triad; hatched bars average over all evaluated pairs that involve at least one shuffled language.
Results are averaged over three disjoint triads.}    \label{fig:correspondence_noise_drop}
\end{figure}
We stress-test alignment under imperfect training correspondences by corrupting the pairing structure used to fit the alignment.
On \textsc{TED-Multi}, we select a language triad and, for each language independently, randomly permute $75\%$ of the training sentences before fitting alignment; evaluation is performed on the unmodified test split.
We consider three disjoint triads (English/French/Spanish; Arabic/Hebrew/Russian; Korean/Japanese/Chinese) and report results averaged over the three trials.

Figure~\ref{fig:correspondence_noise_drop} summarizes the resulting change in cross-lingual rank-1 retrieval accuracy (percentage points) relative to the unshuffled baseline.
The solid bar reports the mean change over evaluated retrieval directions restricted to the triad.
The hatched bar reports a broader aggregate, namely the mean change over all evaluated retrieval directions where at least one side is one of the shuffled languages.
Across both aggregates, GCCA is more sensitive to correspondence noise, while GCPA degrades more gracefully and retains higher retrieval performance under the same corruption level.

Although GPA can show a slightly smaller average drop under this corruption, GCPA attains the highest absolute retrieval accuracy, indicating that the correction remains beneficial even when correspondences are imperfect.

\subsubsection{Person Re-Identification}
\label{sec:reid}
We evaluate cross-camera person re-identification on \textsc{Market-1501}~\cite{zheng2015market} in a zero-shot setting with identity level correspondences. We treat each camera as a separate space and learn an $M$-way alignment across camera-specific representations using training identities only and use the query and gallery sets for retrieval (Appendix~\ref{app:person-reid}). At test time, we perform cross-camera retrieval where queries from camera $C_i$ are matched against galleries from camera $C_j$ ($j\neq i$) using cosine similarity in the aligned coordinates. To reduce camera imbalance, we subsample each camera to match the minimum number of images per identity.

\begin{figure}[htpb!]
    \centering
    \includegraphics[width=\linewidth]{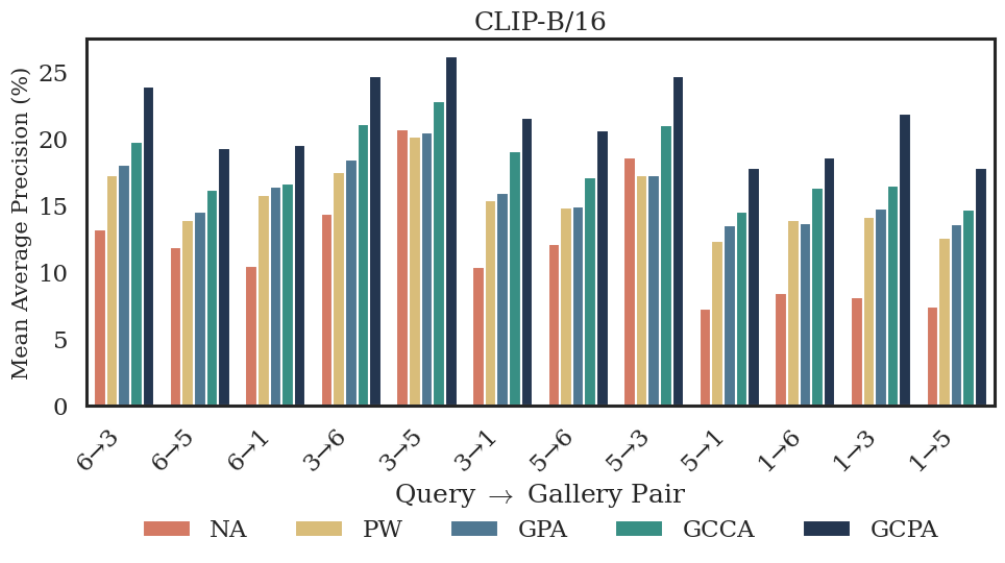}
    \caption{\textbf{Cross-camera retrieval on \textsc{Market-1501} (mAP, \%).} Geometry-Corrected Procrustes Alignment (GCPA) improves over GCCA, GPA, and PW.}
    \label{fig:reid}
\end{figure}

Figure~\ref{fig:reid} presents the retrieval results in terms of Mean Average Precision (mAP). We observe that without alignment (NA), independently trained encoders fail to generalize across views, yielding negligible performance. While standard orthogonal alignments (PW and GPA) successfully restore interoperability, they are consistently outperformed by the agreement-maximizing baseline (GCCA), confirming that strict isometry limits retrieval flexibility. However, our proposed GCPA yields the highest performance across all camera pairs. By correcting the residual directional mismatch within the shared universe, GCPA achieves a distinct margin over GCCA, particularly in transitions such as $6 \to 3$, $3 \to 5$ and $1 \to 3$, demonstrating that combining geometric stability with a consensus-driven correction is crucial for robust cross-view retrieval.

We also conduct an analysis of hyperparameters for cross-camera retrieval and highlight their mild impact on the performance. While we use a fixed GCPA setting across experiments, further performance gains may be obtained by dataset-specific tuning. See Appendix~\ref{app:hyperparams} and~\ref{app:geomchange} for a sensitivity sweep over $(\tau,\lambda)$ and the induced geometry drift and retrieval performance.

\subsubsection{Multimodal alignment}
\label{sec:multimodal}
We study zero-shot cross-modal retrieval on \textsc{Flickr8k}~\cite{harwath2015flickr8k}, which pairs images with text and spoken captions. We construct modality-specific spaces using BERT (text)~\cite{devlin2019bert}, DINOv2 (image)~\cite{caron2021dino}, and HuBERT (audio)~\cite{hsu2021hubert}, alongside CLIP (image)~\cite{radford2021clip} to serve as a strong pivotal anchor. We learn an $M$-way alignment across these modalities, comparing NA, PW, GPA, GCCA, and GCPA. To enforce one-to-one correspondence, we use per-image mean features for text and audio.

\begin{figure}[t]
\centering
\includegraphics[width=\linewidth]{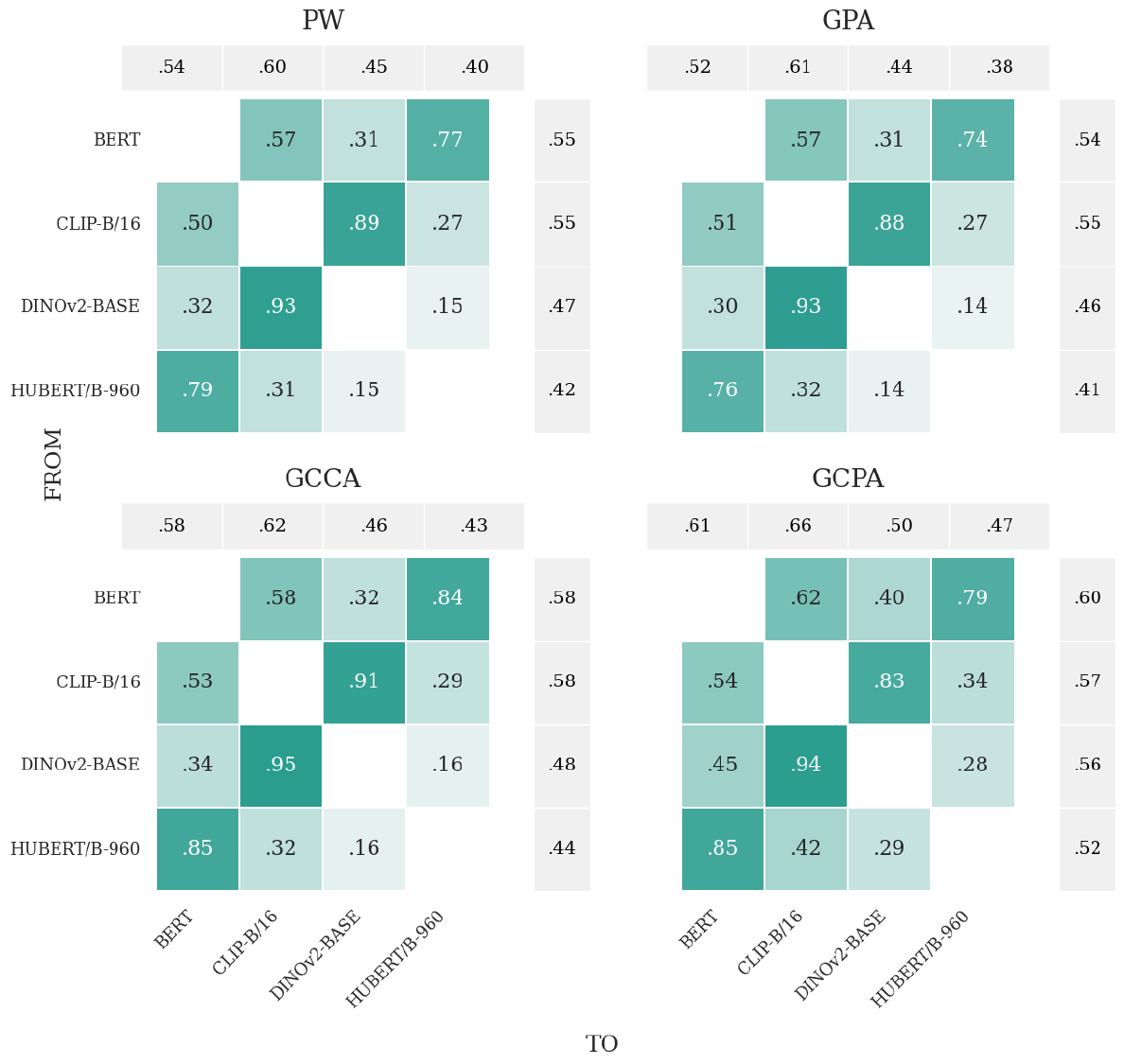}
\caption{\textbf{Rank-1 cross-modal retrieval (\%) on \textsc{Flickr8k}.} GCPA improves audio$\leftrightarrow$ image $\leftrightarrow$text retrieval while retaining the GPA universe.}
\label{fig:semantic-scaffolding-multimodal}
\end{figure}

Figure~\ref{fig:semantic-scaffolding-multimodal} visualizes the pairwise retrieval performance (Rank-1 accuracy) as heatmaps. We observe that the audio modality (HuBERT) presents the most significant alignment challenge, showing weak connectivity to visual models in the standard pairwise (PW) and GPA baselines. For instance, HuBERT retrieval of DINOv2 targets is only $14-15\%$. While GCCA improves some connections, it struggles to bridge the gap between audio and pure vision encoders. In contrast, GCPA yields substantial improvements in these hard settings: HuBERT$\to$DINOv2 retrieval nearly doubles to $29\%$, and HuBERT$\to$CLIP improves from $32\%$ (GPA) to $42\%$ (GCPA). Overall, GCPA achieves the highest column averages across all modalities, demonstrating that the geometric correction effectively pulls audio representations into alignment with the vision-text backbone.

To understand the geometric mechanism behind these gains, we analyze how alignment changes (i) the typical distance between matched cross-modal pairs and (ii) the strength of high-agreement matched triplets. We consider images (I), text (T), and audio (A) modalities, and evaluate all three cross-modal pairs ($I\leftrightarrow T, I \leftrightarrow A, T \leftrightarrow A$).

We use cosine distance on row-wise normalized vectors. Let $(z_i^{I}, z_i^{T}, z_i^{A})$ denote the aligned embeddings for sample $i$ in the three modalities after normalization, and define $d(u,v)=1-\langle u,v\rangle$. For each sample $i$, we define a matched cross-modal distance by averaging the three within-triplet distances:
\begin{align*}
    d_i^{+} &= \tfrac{1}{3}\Big(d(z_i^{I}, z_i^{T}) + d(z_i^{I}, z_i^{A}) + d(z_i^{T}, z_i^{A})\Big),\\
    \Delta^{+} &= \tfrac{1}{N}\sum_{i=1}^{N} d_i^{+}
\end{align*}
where lower $\Delta^{+}$ indicates that matched items are closer on average across modalities.

To summarize three-way agreement, we measure the magnitude of the mean direction of each matched triplet:
\[
\gamma_i \;=\; \left\|\tfrac{1}{3}\big(z_i^{I} + z_i^{T} + z_i^{A}\big)\right\|_2.
\]
We then report a high-agreement summary $\Gamma_{90}$ defined as the value of $\gamma_i$ at the $90$th percentile over the dataset, i.e., the agreement level attained by the $10\%$ most coherent matched triplets. Larger $\Gamma_{90}$ indicates that a substantial fraction of matched triplets achieve strong cross-modal coherence.

Table~\ref{tab:triplet_margin_allinone} reports $\Delta^{+}$ together with $\Gamma_{90}$ across three encoder triplets. The results confirm that GCPA consistently reduces the mean cross-modal distance $\Delta^{+}$ compared to GPA ($0.729 \to 0.560$) and GCCA ($0.744 \to 0.560$), indicating a much tighter clustering of matched modalities. Furthermore, GCPA significantly increases $\Gamma_{90}$ ($0.760 \to 0.853$), showing that the correction successfully concentrates the multimodal universe, creating a high-coherence core that supports improved retrieval performance.

\begin{table}[t]
\centering
\small
\setlength{\tabcolsep}{3pt}
\renewcommand{\arraystretch}{1.05}
\resizebox{.5\textwidth}{!}{
\begin{tabular}{lcccc}
\toprule
& \multicolumn{1}{c}{1} & \multicolumn{1}{c}{2} & \multicolumn{1}{c}{3} & \multicolumn{1}{c}{Mean} \\
\cmidrule(lr){2-2}\cmidrule(lr){3-3}\cmidrule(lr){4-4}\cmidrule(lr){5-5}
Method &
$\Delta^+\downarrow$ \;(\,$\Gamma_{90}\uparrow$\,) &
$\Delta^+\downarrow$ \;(\,$\Gamma_{90}\uparrow$\,) &
$\Delta^+\downarrow$ \;(\,$\Gamma_{90}\uparrow$\,) &
$\Delta^+\downarrow$ \;(\,$\Gamma_{90}\uparrow$\,) \\
\midrule
GPA  & 0.776 (0.732) & 0.689 (0.787) & 0.721 (0.761) & 0.729 (0.760) \\

GCCA & 0.779 (0.728) & 0.730 (0.758) & 0.724 (0.758) & 0.744 (0.748) \\

GCPA (ours) & \textbf{0.623} (\textbf{0.824}) & \textbf{0.503} (\textbf{0.880}) & \textbf{0.554} (\textbf{0.854}) &
\textbf{0.560} (\textbf{0.853}) \\
\bottomrule \\
\end{tabular}}
\caption{\textbf{Agreement analysis (Universe) for multimodal retrieval.}
We report the mean matched cross-modal distance $\Delta^{+}$ (lower is better) together with the upper-tail matched triplet agreement $\Gamma_{90}$ (higher is better), across three encoder triplets.}
\label{tab:triplet_margin_allinone}
\end{table}

\subsection{Robustness and additional analyses}
\label{sec:robustness-extensibility}

The previous sections evaluate GCPA on probing, clustering, and retrieval experiments. We now provide additional analyses of how the shared universe behaves under stochastic fitting, noisy supervision, larger multimodal settings, sequential addition, and changes in the set of aligned spaces.

We first study the reliability of the correction. Repeating the GCPA alignment with different random seeds shows that GCPA remains the best-performing method on \textsc{TED-Multi}, \textsc{Market-1501}, and \textsc{Flickr8k}, with only mild stochasticity in retrieval performance (Appendix~\ref{app:retrieval-sensitivity}). We then
corrupt training correspondences on Flickr8k and evaluate on the clean test split. GCPA remains strongest across different levels of correspondence noise, while all methods degrade as the supervision becomes less reliable (Appendix~\ref{app:flickr8k_noise}). The weak-consensus experiment in Appendix~\ref{app:nonisomorphic} studies the impact of GPA and GCPA alignment when the aligned spaces are too weakly related. In this case, the consensus used by GCPA can itself become unreliable, and GPA becomes the more conservative reference. We also perturb sample embeddings at inference time to test multimodal retrieval and find that GCPA outperforms GPA across varied levels of perturbations (Appendix~\ref{app:ood-sensitivity}).

We next study scaling and extension. Larger multimodal retrieval experiments on \textsc{Flickr30k}~\cite{young2014image} preserve the same qualitative trends as the main multimodal experiments (Appendix~\ref{app:largescaleretrieval}). In continual-addition
experiments, the universe is expanded by adding new spaces sequentially. The GPA universe remains close to full refitting, while GCPA benefits from updating only the correction module after new spaces are added (Appendix~\ref{app:continual-addition}).

We also study how the choice of spaces affects the corrected universe. In this analysis, we fix a base universe and insert candidate spaces one at a time, measuring how the original base models change after each insertion (Appendix~\ref{app:model-selection-stability}). The results show that model capacity alone does not determine whether a candidate is useful. Clean candidates usually have higher similarity to the existing consensus and stronger cross-model probing performance, while corrupted candidates have lower consensus similarity and give smaller gains or slight degradation on the original base models. Thus, candidate selection should consider agreement with the existing consensus, the effect on the original base models, and the candidate's cross-model utility.

Finally, we analyze what GCPA changes after the GPA universe is built. Compared with GPA, the corrected embeddings have lower stable rank, higher effective condition number, and a different leading principal subspace (Appendix~\ref{app:structural_analysis}). In multimodal ablations, the gains cannot be explained solely by the presence of a multimodal anchor such as CLIP (Appendix~\ref{app:multimodal-ablation}). We also compare PCA-based dimension matching with zero-padding for heterogeneous retrieval features in Appendix~\ref{app:multimodal-ablation}.

Together, these analyses clarify where each method is most appropriate. In our experiments, GCPA gives the strongest performance when the aligned spaces provide enough reliable correspondence signal for the consensus correction. When this consensus is unreliable, GPA is the more stable reference. GCCA remains a strong option for retrieval when an explicit reusable universe is not needed.

\section{Conclusions}\label{conclusions}
In this work, we addressed the challenge of aligning multiple representation spaces, demonstrating that standard pairwise strategies fail to scale and lack the consistency required for the $M$-model regime. To overcome these limitations, we proposed a framework that aligns all models into a single shared coordinate system. This formulation reduces alignment complexity from quadratic to linear, facilitates the efficient addition of new models, and enforces cycle consistency by design. Our analysis revealed a critical trade-off in constructing this universe: while Generalized Procrustes Analysis (GPA) yields an orthogonal reference that preserves the internal geometry necessary for model stitching, it is outperformed in zero-shot retrieval by agreement-maximizing methods like GCCA. We resolved this tension with \textbf{Geometry-Corrected Procrustes Alignment (GCPA)}. By using the GPA universe as a robust geometric scaffold and applying a consensus-driven correction, GCPA bridges the gap between geometric fidelity and semantic agreement. Our experiments across multilingual, cross-camera, and multimodal benchmarks confirm that GCPA effectively synthesizes the best of both worlds, achieving state-of-the-art retrieval performance while retaining a practical, composable reference space.

\paragraph{Future Directions.}
A natural extension of this work is to explore more expressive alignment regimes to construct the shared universe. While this study focused on linear and near-linear alignment, future work could investigate $M$-way functional alignment to capture non-isometric correspondences inherent in highly heterogeneous model collections. Furthermore, the ability of the universe to ``heal'' weak links between disparate models may indicate applications in decentralized learning and model merging, where maintaining a coherent global state from fragmented views is critical.

\section*{Impact Statement}
This paper presents work whose goal is to advance the field of machine learning. Shared-universe alignment may support interoperability across independently trained systems, but it may also propagate biases or other undesirable structure across aligned models. We discuss limitations and broader impact in Appendix~\ref{app:limitations}.

% \subsubsection*{Author Contributions}
\section*{Acknowledgments}
This work has been supported by MUR FIS2 grant n. FIS-2023-00942 ``NEXUS'' (cup B53C25001030001) and by Sapienza University of Rome via the Seed of ERC grant ``MINT.AI'' (cup B83C25001040001). The project was ideated as a part of the London Geometry and Machine Learning Summer School (LOGML) 2025. We are grateful to the organizers for hosting the summer school. Akshit Achara was supported by the UK Engineering and Physical Sciences Research Council (EPSRC) [Grant reference number EP/Y035216/1] Centre for Doctoral Training in Data-Driven Health (DRIVE-Health) at King’s College London. Matéo Mahaut received funding from the European Research Council (ERC) under the European Union’s Horizon 2020 research and innovation programme (grant agreement No. 101019291), and from the Catalan government (AGAUR grant SGR 2021 00470). This paper reflects the authors’ view only, and the ERC is not responsible for any use that may be made of the information it contains.

\bibliography{icml2026}
\bibliographystyle{icml2026}

\appendix

\section{Additional details}
\subsection{Multi-Way Alignment Strengthens Weak Pairwise Links} \label{app:weak_links}
To simulate a fragile pair (Section~\ref{sec:weak-links}), we utilize a ResNet-18~\cite{he2016resnet} and a ViT-T/16~\cite{dosovitskiy2020vit} trained on a corrupted version of CIFAR-100~\cite{krizhevsky2009cifar100}. We introduce a distribution shift by replacing $85\%$ of the training images for these two models with binary Canny edge maps~\cite{canny2009computational}, retaining only $15\%$ as original RGB. This degrades the correlation between their learned features compared to models trained on standard data. The anchor models added to the universe are trained on the standard uncorrupted RGB training set. We use a 45{,}000/5{,}000 stratified train/validation split for alignment fitting. Crucially, the corruption is applied only to the training samples used to learn the alignment. The stitching accuracy is reported on the standard, uncorrupted CIFAR-100 test set to evaluate generalization. We also repeat this experiment for GCPA and
Figure~\ref{fig:expanding-set-gcpa} shows that GCPA can underperform when correspondences are weak.
GCPA nudges universe directions toward a per-sample multi-model consensus; if this consensus is distorted by imperfect correspondences (e.g., natural image versus edge-map), the correction can reinforce the mismatch.
As additional models are added, the consensus is better supported and GCPA recovers, surpasses GPA, and continues improving as more anchors are introduced.
This behavior is consistent with the intended role of the correction, which is most beneficial when multi-way agreement reflects shared semantics rather than corruption-specific variation.

Since the GCPA geometry correction is directional, Fig.~\ref{fig:expanding-set-gcpa} reports the unit-direction form to isolate the effect of directional correction in the shared universe coordinates. For scale-sensitive linear probing, we additionally find that rescaling corrected universe coordinates to match the pre-correction GPA norm yields further consistent improvements. Even in this unit-direction form, GCPA can match or exceed the pairwise probing accuracy of the same architecture pair trained on uncorrupted RGB data, and rescaling improves further.

\begin{figure}[t]
    \centering
    \includegraphics[width=0.82\linewidth]{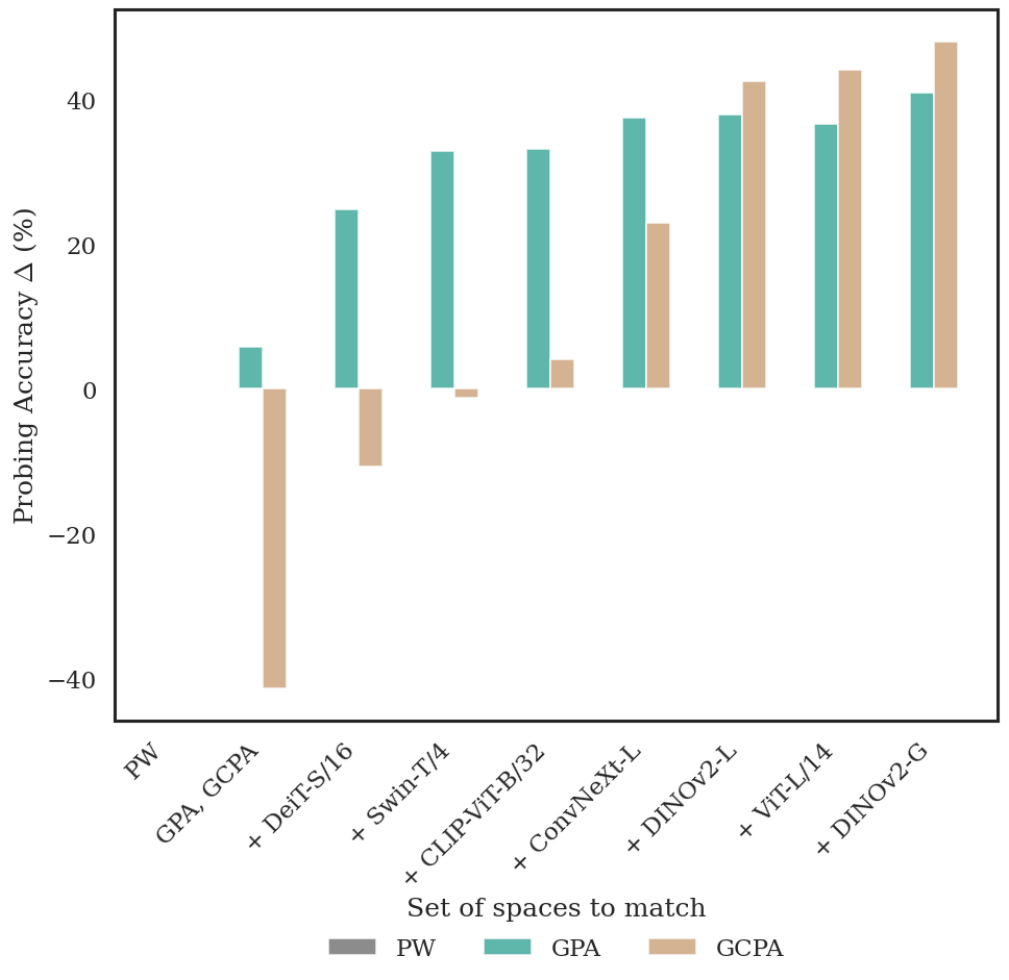}
    \caption{\textbf{Weak-link probing on edge-heavy CIFAR-100 under an expanding alignment set.} We plot the change in probing accuracy (in \%) for the same fragile pair as additional models are added. GPA improves steadily with more anchors, while GCPA can be unstable for very small sets but surpasses GPA once the universe is supported by sufficiently strong models.}
    \label{fig:expanding-set-gcpa}
\end{figure}

\paragraph{When the consensus is unreliable.}
\label{app:nonisomorphic}

The weak-link experiment above shows that GCPA improves once the universe is supported by sufficiently reliable anchors. We now test a more difficult setting in which the universe contains only a fragile pair, ResNet-50 and ViT-B/16, without additional clean anchors. We increase the amount of Canny corruption used to fit the two spaces and evaluate probing accuracy between the two models. This shows that when the aligned spaces are too weakly related, the per-sample consensus used by GCPA can itself become unreliable.

\begin{table}[htpb!]
\centering
\resizebox{.3\textwidth}{!}{
\begin{tabular}{lrrr}
\toprule
Corruption & PW & GPA & GCPA \\
\midrule
50\%  & 0.699 & \textbf{0.719} & 0.643 \\
75\%  & 0.630 & \textbf{0.670} & 0.568 \\
100\% & \textbf{0.036} & 0.035 & 0.015 \\
\bottomrule
\end{tabular}}
\caption{\textbf{Fragile-pair probing accuracy under increasing fraction of Canny images on CIFAR-100.} Values are mean test accuracy over both translation directions.}
\label{tab:nonisomorphic_check}
\vspace{-0.25cm}
\end{table}

Table~\ref{tab:nonisomorphic_check} shows that GPA is more reliable in this setting. At $50\%$ and $75\%$ corruption, GPA remains stronger than pairwise alignment, while GCPA underperforms because the correction is pulled toward a weak two-model consensus. At $100\%$ corruption, all methods collapse close to chance.

Since GCPA acts primarily on directions in the universe, we also tested a rescaled variant that restores the pre-correction GPA norm after applying the directional correction. This improves performance at moderate corruption, but does not address the underlying issue when the consensus itself is unreliable. Thus, in settings with weak consensus support, GPA is the more conservative reference, while GCPA is most useful when the universe has enough clean support for its consensus correction to be meaningful.

\subsection{Incremental addition}
\label{app:add-model}

For a base set of $M$ models and an added model $X_{M+1}$, we report the mean cross-model test accuracy over directed transfers involving the added model:

\[
\begin{aligned}
\mathrm{AvgNew}
&=\frac{1}{2M}\sum_{m=1}^M \Big(
\mathrm{Acc}(X_{M+1}\!\rightarrow X_m) \\
&\hspace{2.2cm}+
\mathrm{Acc}(X_m\!\rightarrow X_{M+1})
\Big).
\end{aligned}
\]

\textbf{PW} fits $\Omega_{m\leftarrow (M+1)}$ and $\Omega_{(M+1)\leftarrow m}$ independently for each $m$.
\textbf{GPA-REFIT} fits $\{\Omega_m\}_{m=1}^{M+1}$ jointly on $\{X_m\}_{m=1}^{M+1}$.
\textbf{GPA-ADD} keeps $\{\Omega_m\}_{m=1}^M$ fixed and learns only $\Omega_{M+1}$, inducing pairwise maps by $\Omega_{m\leftarrow (M+1)}=\Omega_m\Omega_{M+1}^\top$ and $\Omega_{(M+1)\leftarrow m}=\Omega_{M+1}\Omega_m^\top$.
We list all model sets used in \cref{tab:addition-configs}.

\begin{table}[t]
\centering
\small
\setlength{\tabcolsep}{6pt}
\begin{tabular}{lp{0.72\linewidth}}
\toprule
\textbf{Added model} & \textbf{Base model set} \\
\midrule
ViT-B/16
& DINOv3 ViT-L/16, ConvNeXt-L, ResNet-50 \\

Swin-T/4
& DINOv2-B, CLIP ViT-B/32, ViT-B/16 \\

ResNet-18
& ConvNeXt-L, DINOv2-L, CLIP ViT-B/32, ViT-B/16, Swin-L/12 \\

MobileNetV2
& ConvNeXt-B, DINOv3 ViT-L/16, ViT-B/16, CLIP ViT-B/16, ResNet-34 \\
\bottomrule \\

\end{tabular}
\caption{\textbf{Incremental addition configurations.} Each row defines a base model set used to learn a universe, and an added model integrated into that universe.}
\label{tab:addition-configs}
\end{table}

\paragraph{Continual addition.}
\label{app:continual-addition}

We evaluate whether the universe remains stable when models are added
sequentially. We use a class-stratified subset of $2500$ CIFAR-100 training samples. Starting from a universe with $10$ base models, we add $10$ additional models one at a time. The base set contains ResNet-18, ResNet-34, ResNet-50, ViT-T, DeiT-T~\cite{touvron2021training}, DeiT-S, ViT-B/16-224, Swin-T~\cite{liu2021swin}, MobileNetV2, and ConvNeXt-T~\cite{liu2022convnet}. The added models are ConvNeXt-B, DINOv2-S, DINOv2-B, CLIP-ViT-B/16, CLIP-ViT-B/32, ViT-B/16-384, ViT-L/16-224, ConvNeXt-L, DINOv2-L, and Swin-L. After adding $1$, $3$, $6$, $9$, and $10$ models, we compare the incremental universe with a full refit on the same updated model set.

We compare three incremental strategies. \textbf{GPA-ADD} keeps all existing GPA maps fixed and learns only the orthogonal map from the added model into the current mean universe. \textbf{GCPA-ADD} uses the same add-only map update and keeps the existing GCPA correction module fixed. \textbf{GCPA-ADD-LIGHT} also uses the same add-only update as GCPA-ADD, and then re-optimizes the GCPA correction module on the enlarged set of universe representations. Thus, it leaves the existing GPA maps fixed while updating only the inserted-space alignment and the correction stage.

\begin{table}[htpb!]
\centering
\resizebox{.45\textwidth}{!}{
\begin{tabular}{lcc}
\toprule
\textbf{Incremental strategy} &
\textbf{Mean diff. from refitting} &
\textbf{Final accuracy} \\
\midrule
GPA-ADD & $-0.002$ & $0.626$ \\
GCPA-ADD & $-0.021$ & $0.624$ \\
GCPA-ADD-LIGHT & $-0.009$ & \textbf{$0.643$} \\
\bottomrule
\end{tabular}}
\caption{\textbf{Stability under continual addition on CIFAR-100
probing.} Starting from $10$ models, we add $10$ more models sequentially and compare incremental addition against full refitting after $1$, $3$, $6$, $9$, and $10$ additions. Final accuracy is measured over all cross-model pairs after the universe contains $20$ models.}
\label{tab:continual_addition}
\vspace{-0.25cm}
\end{table}

The test set evaluation in Table~\ref{tab:continual_addition} shows that GPA-ADD stays very close to GPA-REFIT, with a mean difference of $-0.002$ over the evaluated addition steps at 1, 3, 6, 9, and 10 added models. This indicates that the orthogonal reference remains stable when older maps are kept fixed. GCPA-ADD has a larger gap to GCPA-REFIT, suggesting that the fixed correction module becomes less well matched as new spaces are added. Re-optimizing only the correction module reduces this gap from $-0.021$ to $-0.009$ and gives the highest final accuracy after all $20$ models are included. Thus, sequential extension does not require refitting the full universe, although GCPA benefits from occasional updates.

\subsection{Person Re-Identification}\label{app:person-reid}
\textsc{Market-1501} provides images from six cameras. We select four cameras with the largest identity overlap, yielding $437$ training identities and $465$ identities for both query and gallery. Query and gallery identities coincide and are disjoint from train. Each selected camera defines one space ($C_1,C_2,C_3,C_4$). We learn the multi-way alignment using per-image embeddings computed on the training split. For evaluation, we subsample each camera to match the per-identity minimum number of images in query and gallery, and assess all cross-camera pairs $C_i \mapsto C_j$ with $i\neq j$. Note that this dataset does not provide cross‑camera image‑level correspondences and therefore, our evaluation aligns person identities with no assumption of per‑image pairing across cameras.

\subsection{Multilingual encoder set}
\label{app:multilingual-encoders}

We use one language-specific Transformer encoder per language.
This encoder set is used across our multilingual experiments, including the multilingual retrieval (Section~\ref{sec:tedmulti}), correspondence-noise stress test (Section~\ref{subsec:corrupt-correspondence}) and the massive multilingual clustering experiment (Section~\ref{sec:universal-clustering}).

\begin{table}[t]
\centering
\small
\setlength{\tabcolsep}{4pt}
\begin{tabular}{l l}
\toprule
Language & Encoder \\
\midrule
English (EN)   & RoBERTa-base \\
French (FR)    & CamemBERT-base \\
Spanish (ES)   & Spanish BERT-base \\
Italian (IT)   & Italian BERT-base \\
Arabic (AR)    & AraBERTv2 (base) \\
Hebrew (HE)    & HeBERT (base) \\
Russian (RU)   & RuBERT (base) \\
Japanese (JA)  & Japanese BERT-base \\
Korean (KO)    & Korean BERT-base \\
Chinese (ZH)   & Simplified Chinese BERT-base \\
\bottomrule
\end{tabular}
\caption{\textbf{Language-specific encoder set used in our multilingual experiments.} Each row contains the language and the corresponding encoder used in our multilingual experiments.}
\label{tab:multilingual-encoders}
\vspace{-4mm}
\end{table}

\subsection{Cross-modal protocol details}\label{app:cross-modal}
Flickr8k contains 8{,}000 images, each paired with five text and five spoken captions. We construct modality-specific spaces (image, text, audio) and form per-image mean representations for text and audio to obtain one-to-one correspondence with image features. We learn an $M$-way alignment with one space per modality and evaluate zero-shot cross-modal retrieval by querying from one modality and retrieving in another (e.g., image$\to$audio, audio$\to$image).

\paragraph{Multimodal retrieval at a larger scale.}
\label{app:largescaleretrieval}

We also evaluate multimodal retrieval on \textsc{Flickr30k}. We use $6000$ image--caption pairs for training and $1000$ pairs for testing, selecting one caption per image. We evaluate cross-encoder retrieval.

We construct a universe with five image (DINOv2-B, DINOv2-L, CLIP-B/32, CLIP-L/14, ConvNeXt-L) and five text spaces (CLIP-B/32, CLIP-L/14, MiniLM~\cite{wang2020minilm}, MPNet~\cite{song2020mpnet}, and BGE-small~\cite{bge_embedding}). We evaluate two training conditions. In the clean setting, the original training correspondences are used. In the noisy setting, we corrupt $25\%$ of the training correspondences for the text spaces by replacing each selected correspondence with its nearest neighbor in the CLIP-L/14 vision embedding space. Evaluation is performed on the same clean test set in both cases. Here, we use a lighter GCPA correction setting to make the correction more conservative.

Table~\ref{tab:largescaleretrieval} shows that GCPA performs best across all reported metrics in both clean and noisy settings. The noisy setting reduces performance for all methods, but GCPA retains the highest average retrieval and worst-pair retrieval performance.

\begin{table}[htpb!]
\centering
\resizebox{.4\textwidth}{!}{
\begin{tabular}{llrrrr}
\toprule
Condition & Method & Mean Cos. & R@1 & mAP & Worst R@1 \\
\midrule
\multirow{4}{*}{Clean}
& PW   & 0.257 & 0.602 & 0.686 & 0.164 \\
& GPA  & 0.287 & 0.601 & 0.688 & 0.178 \\
& GCCA & 0.305 & 0.626 & 0.709 & 0.189 \\
& GCPA & \textbf{0.486} & \textbf{0.640} & \textbf{0.733} & \textbf{0.270} \\
\midrule
\multirow{4}{*}{Noisy}
& PW   & 0.250 & 0.585 & 0.668 & 0.148 \\
& GPA  & 0.280 & 0.588 & 0.674 & 0.161 \\
& GCCA & 0.298 & 0.607 & 0.689 & 0.168 \\
& GCPA & \textbf{0.465} & \textbf{0.613} & \textbf{0.710} & \textbf{0.239} \\
\bottomrule
\end{tabular}}
\caption{\textbf{Multimodal retrieval under clean and noisy correspondences on \textsc{Flickr30k}.} Values are mean cosine similarity, R@1, mAP, and worst-pair R@1; higher is better.}
\label{tab:largescaleretrieval}
\vspace{-0.25cm}
\end{table}

\subsection{Computational complexity and costs}
\label{app:computational-complexity}

As GCPA adds a correction step on top of the GPA universe fit and the per-sample consensus direction is formed by averaging the aligned directions across the $M$ spaces, this costs $\mathcal{O}(Md)$ per sample, or $\mathcal{O}(BMd)$ for a batch of size $B$. Across $N$ samples, the consensus can be formed in $\mathcal{O}(NMd)$ time and stored in $\mathcal{O}(Nd)$ memory.

We measure runtime and peak memory on synthetic matched representations with $d=256$ on a single NVIDIA RTX A6000 system. We fit each method three times and report the average runtime and peak memory. In the view-scaling experiment with $N=10{,}000$ fixed, the full GCPA fit uses about $5\times$, $3\times$, and $2\times$ the runtime of GPA at $M=10$, $20$, and $40$, respectively, and about $1.5\times$ to $1.7\times$ the memory of GPA. GCCA is faster than GPA and GCPA for smaller numbers of spaces, but uses substantially more memory. Across the same experiment, GCCA uses about $4\times$ to $6\times$ the memory of GPA and about $2\times$ to $4\times$ the memory of GCPA.

When we vary the number of samples from $N=1000$ to $N=20{,}000$ with $M=10$ fixed, memory usage increases as expected. For incremental extension, adding one new space to a fixed universe is substantially cheaper than full refitting after insertion, taking about $5\%$ of the refitting time in our fixed GCPA training setup.

\subsection{Retrieval variability across GCPA fits}
\label{app:retrieval-sensitivity}

GCPA includes a stochastic correction step, so we repeat the main retrieval experiments with $10$ random seeds for GCPA. We evaluate the
same settings as in the main text: \textsc{TED-Multi} (Section~\ref{sec:tedmulti}), \textsc{Market-1501}
(Section~\ref{sec:reid}), and \textsc{Flickr8k}
(Section~\ref{sec:multimodal}). We report the mean and $95\%$ confidence interval for GCPA. GCCA is deterministic in our setup, so we report its point estimate.

\begin{table}[htpb!]
\centering
\resizebox{.45\textwidth}{!}{
\begin{tabular}{lccc}
\toprule
\textbf{Benchmark} & \textbf{Metric} & \textbf{GCPA mean [95\% CI]} & \textbf{GCCA} \\
\midrule
\textsc{TED-Multi}    & $R@1$ & 0.502 $[0.500, 0.505]$ & 0.487 \\
\textsc{Market-1501}  & $R@1$ & 0.178 $[0.171, 0.185]$ & 0.154 \\
\textsc{Market-1501}  & mAP   & 0.199 $[0.191, 0.207]$ & 0.180 \\
\textsc{Flickr8k}     & $R@1$ & 0.550 $[0.545, 0.554]$ & 0.521 \\
\bottomrule
\end{tabular}
}
\caption{\textbf{Retrieval performance across repeated GCPA alignment.} Values are $R@1$ and mean average precision where higher indicates better retrieval performance.}
\label{tab:retrieval_ci}
\vspace{-0.25cm}
\end{table}

Table~\ref{tab:retrieval_ci} shows that the stochasticity of the GCPA correction does not change the qualitative ranking and GCPA remains above GCCA on all three retrieval experiments.

\section{Theorem proofs}
\label{app:proofs}

Here, we present some theoretical properties of GCCA which are useful for retrieval and why consensus based correction is an effective way to correct directional mismatch.

\begin{theorem}[Optimal multi-space alignment under squared cross-space discrepancy]
Let $S \in \mathbb{R}^{(\sum_i r_i) \times (\sum_i r_i)}$ be the block matrix with diagonal blocks
$S_{ii} = (M-1)I_{r_i}$ and off-diagonal blocks $S_{ij} = -U_i^\top U_j$.
Then the global minimizer of Eq.~\ref{eq:gcca_objective} is given by the eigenvectors corresponding to the $R$ smallest eigenvalues of $S$, stacked into $\Phi^\star = [\Phi_1^\star; \dots; \Phi_M^\star]$.
Moreover, the objective in Eq.~\ref{eq:gcca_objective} admits the equivalent trace form:
\[
\sum_{i<j} \|U_i \Phi_i - U_j \Phi_j\|_F^2
=
\mathrm{Tr}(\Phi^\top S \Phi),
\]
where $\Phi = [\Phi_1; \dots; \Phi_M]$.
\label{thm:gcca_solution}
\end{theorem}

\begin{proof}
Recall the objective in Eq.~\ref{eq:gcca_objective}:
\[
\min_{\{\Phi_i\}}
\sum_{1 \le i < j \le M} \|U_i\Phi_i - U_j\Phi_j\|_F^2
\quad \text{s.t.}\quad
\sum_{i=1}^M \Phi_i^\top \Phi_i = I_R,
\]
where each $U_i \in \mathbb{R}^{N\times r_i}$ has orthonormal columns ($U_i^\top U_i = I_{r_i}$).
Define the stacked variable
\[
\Phi := 
\begin{bmatrix}
\Phi_1\\ \vdots\\ \Phi_M
\end{bmatrix}
\in \mathbb{R}^{(\sum_i r_i)\times R}.
\]

First, we prove that the objective can be rewritten as a trace quadratic form.
For any $i<j$,
\begin{align*}
\|U_i\Phi_i - U_j\Phi_j\|_F^2
&= \mathrm{Tr}\!\left((U_i\Phi_i - U_j\Phi_j)^\top (U_i\Phi_i - U_j\Phi_j)\right) \\
&= \mathrm{Tr}(\Phi_i^\top U_i^\top U_i \Phi_i)
+ \mathrm{Tr}(\Phi_j^\top U_j^\top U_j \Phi_j) \\
&- 2\,\mathrm{Tr}(\Phi_i^\top U_i^\top U_j \Phi_j) \\
&= \|\Phi_i\|_F^2 + \|\Phi_j\|_F^2
- 2\,\mathrm{Tr}(\Phi_i^\top U_i^\top U_j \Phi_j),
\end{align*}
using $U_i^\top U_i = I_{r_i}$ and $U_j^\top U_j = I_{r_j}$.

Summing over all pairs $1\le i<j\le M$, each term $\|\Phi_i\|_F^2$ appears exactly $(M-1)$ times, hence
\begin{align*}
\sum_{i<j}\|U_i\Phi_i - U_j\Phi_j\|_F^2
&= \sum_{i=1}^M (M-1)\|\Phi_i\|_F^2 \\
&- 2\sum_{i<j}\mathrm{Tr}(\Phi_i^\top U_i^\top U_j \Phi_j).
\end{align*}

Now define the symmetric block matrix $S \in \mathbb{R}^{(\sum_i r_i)\times(\sum_i r_i)}$ with blocks
\[
S_{ii}=(M-1)I_{r_i}, \qquad S_{ij}=-U_i^\top U_j \;\;(i\neq j).
\]
A direct block expansion yields the identity
\[
\mathrm{Tr}(\Phi^\top S \Phi)
=
\sum_{i=1}^M (M-1)\|\Phi_i\|_F^2
- 2\sum_{i<j}\mathrm{Tr}(\Phi_i^\top U_i^\top U_j \Phi_j),
\]
which matches the expression above. Therefore,
\[
\sum_{i<j}\|U_i\Phi_i - U_j\Phi_j\|_F^2
=
\mathrm{Tr}(\Phi^\top S \Phi).
\]
The constraint $\sum_i \Phi_i^\top \Phi_i = I_R$ is exactly $\Phi^\top \Phi = I_R$.

Now let us solve the constrained trace minimization.
We have reduced Eq.~\ref{eq:gcca_objective} to
\[
\min_{\Phi^\top \Phi = I_R}\ \mathrm{Tr}(\Phi^\top S \Phi).
\]
Since $S$ is symmetric, by the Rayleigh--Ritz/Ky Fan variational characterization, the minimum is attained when the columns of $\Phi$ span the eigenspace associated with the $R$ smallest eigenvalues of $S$. Equivalently, an optimal solution $\Phi^\star$ is obtained by taking as columns the eigenvectors corresponding to the $R$ smallest eigenvalues of $S$ (orthonormalized), and partitioning $\Phi^\star$ into blocks $\{\Phi_i^\star\}_{i=1}^M$.

This proves that the global minimizer of Eq.~\ref{eq:gcca_objective} is given by the bottom-$R$ eigenvectors of $S$, and the objective admits the trace form claimed in the theorem.
\end{proof}

\begin{corollary}
Let $Y_i^\star = U_i \Phi_i^\star$ denote the aligned embeddings obtained from Theorem~\ref{thm:gcca_solution}. Among all embeddings representable within the retained subspaces $\{U_i\}$ and satisfying the global orthonormality constraint, $\{Y_i^\star\}$ minimize the total cross-space mismatch energy.
\label{corr:1}
\end{corollary}

\begin{proof}
By definition, the total cross-space mismatch energy is exactly the objective in Eq.~\ref{eq:gcca_objective},
\begin{align*}
\sum_{1\le i<j\le M}\|Y_i - Y_j\|_F^2
=
\sum_{1\le i<j\le M}\|U_i\Phi_i - U_j\Phi_j\|_F^2, \\
\qquad Y_i=U_i\Phi_i.
\end{align*}
The feasible set in Corollary~\ref{corr:1} coincides with the feasible set of Eq.~\ref{eq:gcca_objective},
namely all collections $\{\Phi_i\}$ satisfying the global orthonormality constraint
$\sum_{i=1}^M \Phi_i^\top \Phi_i = I_R$ (equivalently $\Phi^\top\Phi=I_R$ for $\Phi=[\Phi_1;\dots;\Phi_M]$).
Theorem~\ref{thm:gcca_solution} states that $\{\Phi_i^\star\}$ is a \emph{global minimizer} of Eq.~\ref{eq:gcca_objective}.
Therefore, the corresponding embeddings $Y_i^\star = U_i\Phi_i^\star$ globally minimize the mismatch energy
over the stated admissible class. 
\end{proof}

\begin{corollary}
Let $G := [Y_1^\star | \dots | Y_M^\star] \in \mathbb{R}^{N \times MR}$. The matrix $B \in \mathbb{R}^{N \times R}$ formed by the top $R$ left singular vectors of $G$ provides an orthonormal basis for the dominant shared latent subspace induced by the aligned spaces.
\label{corr:2}
\end{corollary}

\begin{proof}
Let $G := [Y_1^\star \mid \dots \mid Y_M^\star] \in \mathbb{R}^{N\times MR}$.
We consider the optimization problem
\[
\max_{B^\top B = I_R}\ \|B^\top G\|_F^2.
\]
Observe that
\[
\|B^\top G\|_F^2
=
\mathrm{Tr}\!\left((B^\top G)(B^\top G)^\top\right)
=
\mathrm{Tr}\!\left(B^\top G G^\top B\right).
\]
Define the symmetric positive semidefinite matrix $A := GG^\top \succeq 0$.
Then the problem becomes
\[
\max_{B^\top B = I_R}\ \mathrm{Tr}(B^\top A B).
\]
By the Ky Fan maximum principle (equivalently the variational characterization of the sum of the top eigenvalues),
the maximizer $B$ is given by the eigenvectors of $A$ associated with its largest $R$ eigenvalues.
Since the eigenvectors of $GG^\top$ are precisely the left singular vectors of $G$,
the solution is the matrix formed by the top $R$ left singular vectors of $G$.
\end{proof}

\begin{corollary}
The corresponding linear maps from the original feature spaces into the shared latent space are given by $Q_i := V_i \Phi_i^\star$, so that for any feature vector $z \in \mathbb{R}^{d_i}$, the embedding $z Q_i$ realizes the same aligned coordinates as the sample-side embedding $U_i \Phi_i^\star$, up to the scaling induced by the SVD convention. Equivalently, defining $Q_i := V_i \Sigma_i^{-1} \Phi_i^\star$ yields $X_i Q_i = U_i \Phi_i^\star$ exactly.
\label{corr:3}
\end{corollary}

\begin{proof}
Let $X_i = U_i\Sigma_i V_i^\top$ be a (possibly truncated) SVD with $U_i^\top U_i = I$ and $V_i^\top V_i = I$.
First consider the map $\widetilde{Q}_i := V_i \Sigma_i^{-1}\Phi_i^\star$ (assuming $\Sigma_i$ is invertible on the retained rank).
Then
\[
X_i \widetilde{Q}_i
=
(U_i\Sigma_i V_i^\top)(V_i\Sigma_i^{-1}\Phi_i^\star)
=
U_i \Phi_i^\star
=
Y_i^\star,
\]
which shows that $\widetilde{Q}_i$ yields an \emph{exact} feature-space realization of the aligned sample embedding.

For the map $Q_i := V_i\Phi_i^\star$, we similarly obtain
\[
X_i Q_i
=
(U_i\Sigma_i V_i^\top)(V_i\Phi_i^\star)
=
U_i\Sigma_i\Phi_i^\star.
\]
Thus, $Q_i$ realizes the same aligned coordinates up to the singular-value scaling induced by the SVD convention
(i.e., absorbing $\Sigma_i$ into the coordinates). This is the stated ``up to scaling'' relationship.
\end{proof}

\paragraph{Proof of Proposition~\ref{prop:consensus}.}
For unit vectors $\{\widehat{u}_{m,i}\}_{m=1}^M$, let $s := \sum_{m=1}^M \widehat{u}_{m,i}$ and
$c_i = s/\|s\|_2$. Then
\[
\frac{1}{M}\sum_{m=1}^M \langle \widehat{u}_{m,i}, c_i \rangle
= \frac{1}{M}\left\langle \sum_{m=1}^M \widehat{u}_{m,i}, \frac{s}{\|s\|_2}\right\rangle
= \frac{1}{M}\|s\|_2,
\]
which gives Eq.~\ref{eq:mean_to_consensus}. Moreover,
\begin{align*}
\|s\|_2^2 = \sum_{m=1}^M \|\widehat{u}_{m,i}\|_2^2 + 2\sum_{m<n}\langle \widehat{u}_{m,i}, \widehat{u}_{n,i}\rangle
= \\
M + 2\sum_{m<n}\langle \widehat{u}_{m,i}, \widehat{u}_{n,i}\rangle,
\end{align*}
which rearranges to Eq.~\ref{eq:pairwise_sum_identity}.
\hfill$\square$

\section{Ablations}
\label{app:geom_corr}
\subsection{Hyperparameters}
\label{app:hyperparams}
The trust penalty in GCPA acts as a soft trust-region as it allows the corrector to nudge a point toward the per-sample consensus direction, while discouraging excessive deviation from the original GPA universe direction.
In our implementation, $\tau$ specifies a tolerance on drift (no penalty is incurred below this threshold), and $\lambda$ controls the strength of the penalty once the drift exceeds $\tau$.

Rather than tuning these parameters per benchmark, we fix a single conservative $(\tau,\lambda)$ for all experiments.
The intent is to allow mild corrections (so retrieval can benefit from consensus nudging) while maintaining a clear notion of geometric trust in the underlying GPA universe.
To make the effect of these parameters transparent, we report a sensitivity sweep on \textsc{Market-1501} in \cref{fig:gcpa_hparam_heatmap}. It can be seen that there are no abrupt performance differences and mild settings can help in achieving a reasonable retrieval performance.

\begin{figure}[htpb!]
    \centering
    \includegraphics[width=\linewidth]{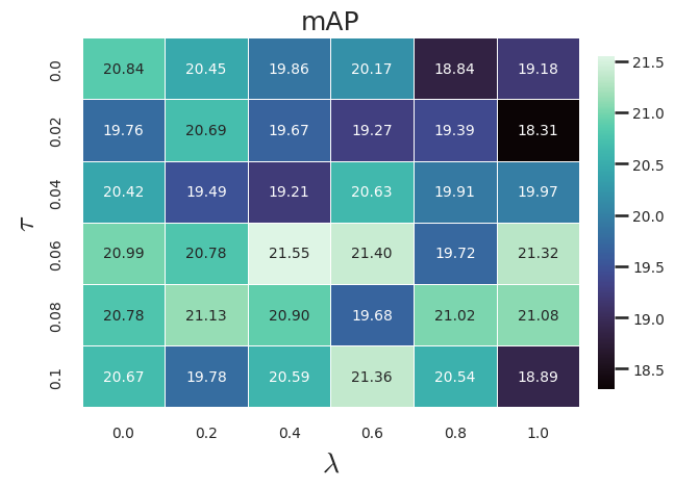}
    \caption{\textbf{Sensitivity of GCPA to the trust penalty on cross-camera retrieval.}
    We sweep the trust-region parameters $(\tau,\lambda)$ and report cross-camera mAP (\%, higher is better).}
    \label{fig:gcpa_hparam_heatmap}
\end{figure}

\subsection{Geometry changes on nudging}
\label{app:geomchange}

GCPA improves retrieval by applying a small correction in the GPA universe, which intentionally relaxes strict isometry.
To make this trade-off explicit, we quantify how much the correction changes the geometry of the universe embeddings.

Let $\hat{u}$ denote the row-wise $\ell_2$-normalized universe embedding produced by GPA, and let $\tilde{u}$ be the corresponding corrected embedding after applying the GCPA update (and renormalization).

Specifically, we measure drift as the angular deviation between $\tilde{u}$ and $\hat{u}$, using $\mathbb{E}[1-\langle \tilde{u}, \hat{u}\rangle]$.

Figure~\ref{fig:geom-change-drift} summarizes the drift percentage for different values of $\lambda$ and $\tau$ providing insights on how increasing $\lambda$ results in a decrease in drift whereas increase $\tau$ results in an increase.
Overall, GCPA introduces a controlled amount of drift and this is expected as unlike GPA, which is exactly isometric, GCPA is designed to allow limited geometric deformation when it increases cross-view agreement and improves retrieval. We fit the alignment using three different text encoders for this computation.

\begin{figure}[htpb!]
    \centering
    \includegraphics[width=\linewidth]{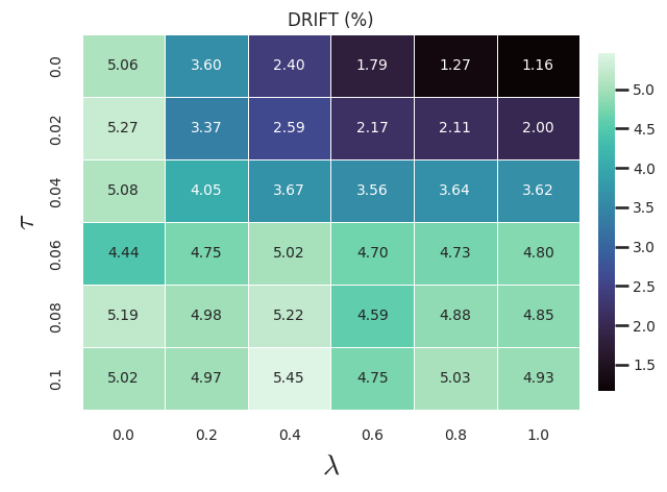}
    \caption{\textbf{Sensitivity of GCPA to the trust penalty on drift.}
    We sweep the trust-region parameters $(\tau,\lambda)$ and report median drift (\%, lower is better). The top-right corner consists of higher $\lambda$ and lower $\tau$ and therefore has lower drift values.}
    \label{fig:geom-change-drift}
\end{figure}

\subsection{Structure of the GCPA universe}
\label{app:structural_analysis}

We analyze how the GCPA correction changes the shared universe. We fit the
alignment on the training split and map each model's embeddings into the shared universe. We row-normalize the mapped embeddings, concatenate them across models, center the resulting matrix, and compute its singular values. Let $\sigma_1 \geq \sigma_2 \geq \cdots$ denote these singular values.

We report four quantities. First, $k_{95}$ is the smallest number of leading singular directions whose squared singular values account for at least $95\%$ of the sum of all squared singular values. Second, the effective condition number is $\sigma_1/\sigma_{k_{95}}$. Third, the stable rank is $\sum_i \sigma_i^2/\sigma_1^2$. Finally, we measure the overlap between the subspace spanned by the top $25$ principal directions of the GPA universe and the corresponding subspace for the GCPA universe. Lower overlap means that the correction changes the dominant directions of the shared space.

We first evaluate the 10-model \textsc{TED-Multi} setting used in the main experiments. On the test split, after GCPA correction, $k_{95}$ decreases from $518$ under GPA to $321$. The effective condition number increases from about $11.3$ to $16$, and the stable rank decreases from about $28.3$ to $15.2$. Thus, after correction, fewer singular directions are needed to account for most of the variation in the shared universe, and the leading directions explain a larger share of the representation structure.

We observe the same pattern on \textsc{Flickr8k}. We consider multimodal settings with HuBERT-base or HuBERT-large for audio, BERT for text, and either DINOv2-base or ResNet-18 for vision. On the test split, $k_{95}$ drops from $416$ to $308$ for HuBERT-base/BERT/DINOv2-base, from $430$ to $395$ for HuBERT-base/BERT/ResNet-18, from $558$ to $466$ for HuBERT-large/BERT/DINOv2-base, and from $429$ to $381$ for HuBERT-large/BERT/ResNet-18. These changes are again accompanied by a higher effective condition number and lower stable rank after correction.

The overlap between the top-$25$ GPA and GCPA subspaces is low, at most about $0.05$ across these settings. This shows that GCPA is not simply rescaling the same dominant directions found by GPA. Instead, the correction changes the leading structure of the shared universe while remaining anchored to the GPA reference through its correction objective. Overall, GCPA produces a more concentrated and more anisotropic universe, with dominant directions that differ substantially from those of the original GPA universe.

\subsection{Noisy-correspondence robustness on Flickr8k}
\label{app:flickr8k_noise}

We further evaluate robustness to corrupted training correspondences on
\textsc{Flickr8k} in a 10-space multimodal setting. The universe contains three audio spaces, HuBERT-B/960, HuBERT-L/960, and Wav2Vec2-B/960~\cite{baevski2020wav2vec}; two text spaces, BERT and ALBERT~\cite{lan2020albert}; and five image spaces, DINOv2-B, ResNet-18, ViT-B/16, ConvNeXt-B, and ResNet-50.

We corrupt the training correspondences for a subset of spaces before fitting the alignment, and evaluate mean cross-modal retrieval on the original test split. The corrupted spaces are HuBERT-B/960, BERT, ALBERT, ResNet-18, and ViT-B/16. This evaluates how the aligned universe behaves when part of the supervision used to construct it becomes unreliable.

\begin{table}[htpb!]
\centering
\resizebox{.3\textwidth}{!}{
\begin{tabular}{lccc}
\toprule
Corruption & GPA & GCCA & GCPA \\
\midrule
25\% & 0.266 & 0.296 & \textbf{0.347} \\
50\% & 0.199 & 0.219 & \textbf{0.271} \\
75\% & 0.091 & 0.110 & \textbf{0.126} \\
90\% & 0.041 & 0.052 & \textbf{0.056} \\
\bottomrule
\end{tabular}
}
\caption{\textbf{Noisy-correspondence robustness on multimodal retrieval.} Values represent $R@1$ cross-modal retrieval performance.}
\label{tab:flickr8k_noise}
\vspace{-0.25cm}
\end{table}

Table~\ref{tab:flickr8k_noise} shows that GCPA achieves the highest retrieval performance at every corruption level. The gap is largest at moderate corruption levels, where enough correct correspondence signal remains for the consensus correction to improve the aligned universe. As the corruption level increases, all methods degrade and the gap narrows, consistent with the fact that a severely corrupted training signal provides less reliable structure for alignment. Even at $75\%$ and $90\%$ corruption, GCPA remains above both GPA and GCCA.

\subsection{Sensitivity to noisy test embeddings}
\label{app:ood-sensitivity}

The correspondence-noise and weak-link experiments perturb the supervision used to fit the alignment. Here, we instead keep the learned universe fixed and perturb only the test embeddings. This isolates the effect of noisy sample features at inference time.

We run this experiment on \textsc{Flickr8k} using HuBERT-B/960 for audio, BERT for text, and DINOv2-B for vision. For each view, we add Gaussian noise to the test embeddings. The noise is scaled per feature dimension by the corresponding training-set standard deviation, and we vary the noise level from $0$ to $1$. We then map the perturbed embeddings into the fixed learned universe and measure cross-encoder retrieval. We also report drift, defined as the mean value of $1-$ cosine similarity between the original and perturbed test representations after projection into the shared universe.

\begin{table}[hptb!]
\centering
\resizebox{.4\textwidth}{!}{
\begin{tabular}{llrrr}
\toprule
Method & Noise & Avg. R@1 & Worst R@1 & Drift \\
\midrule
\multirow{6}{*}{GPA}
& 0.00 & 0.373 & 0.155 & 0.000 \\
& 0.05 & 0.372 & 0.153 & 0.001 \\
& 0.10 & 0.372 & 0.157 & 0.005 \\
& 0.20 & 0.369 & 0.156 & 0.020 \\
& 0.50 & 0.325 & 0.130 & 0.108 \\
& 1.00 & 0.197 & 0.068 & 0.297 \\
\midrule
\multirow{6}{*}{GCPA}
& 0.00 & 0.424 & 0.186 & 0.000 \\
& 0.05 & 0.424 & 0.181 & 0.001 \\
& 0.10 & 0.423 & 0.186 & 0.003 \\
& 0.20 & 0.413 & 0.178 & 0.013 \\
& 0.50 & 0.365 & 0.152 & 0.073 \\
& 1.00 & 0.234 & 0.102 & 0.226 \\
\bottomrule
\end{tabular}
}
\caption{\textbf{Effect of Gaussian feature noise on the learned universe.} We report cross-encoder retrieval performance and mean drift from the original test projection into the universe.}
\label{tab:gaussian_universe_noise}
\vspace{-0.25cm}
\end{table}

Table~\ref{tab:gaussian_universe_noise} shows that retrieval degrades gradually as the feature noise increases, and the drift in the shared universe increases accordingly. GCPA remains stronger than GPA at every noise level. It also shows lower drift than GPA under the same perturbation level, indicating that the correction does not make the learned universe more sensitive to this form of test-time feature noise.

\subsection{Multimodal ablations}
\label{app:multimodal-ablation}

\paragraph{Effect of multimodal anchors.}
We examine whether the multimodal retrieval gains of GCPA are primarily due to including a multimodal encoder, such as CLIP, in the alignment universe. To test this, we run an ImageNet-1k~\cite{russakovsky2015imagenet} class-prompt
retrieval experiment with and without CLIP in the universe. We use 6000 images as the training set and 1000 as the test set with class stratification. For each ImageNet class, we construct a text prompt of the form ``a photo of a \{class\}'' and embed these prompts using BERT~\cite{devlin2019bert}. We then evaluate image$\leftrightarrow$text retrieval after alignment.

We fix the text encoder to BERT and compare five universe configurations:
(A) BERT+CLIP-B/32+DINOv2-B, (B) BERT+CLIP-L/14+DINOv2-B,
(C) BERT+DINOv2-B, (D) BERT+DINOv2-B+ViT-B/16, and
(E) BERT+DINOv2-B+ViT-B/16+RN50.

\begin{table}[htpb!]
\centering
\resizebox{.35\textwidth}{!}{
\begin{tabular}{lrrrr}
\toprule
Setting & PW & GPA & GCCA & GCPA \\
\midrule
A & 0.596 & 0.590 & 0.592 & \textbf{0.646} \\
B & 0.559 & 0.557 & 0.611 & \textbf{0.699} \\
C & 0.649 & 0.648 & 0.676 & \textbf{0.717} \\
D & 0.665 & 0.657 & 0.704 & \textbf{0.740} \\
E & 0.681 & 0.673 & 0.705 & \textbf{0.744} \\
\bottomrule
\end{tabular}
}
\caption{\textbf{R@1 on the ImageNet class-prompt retrieval experiment.} Settings A and B include CLIP in the universe, while C, D, and E do not.}
\label{tab:clipabl}
\vspace{-0.25cm}
\end{table}

Table~\ref{tab:clipabl} shows that GCPA achieves the best performance for all combinations of models. The strongest results are obtained in settings without CLIP, showing that the gains are not simply inherited from a pretrained multimodal anchor. The results instead point to the importance of how well the selected spaces support a useful shared consensus.

The previous ablation shows that the gains do not require CLIP to be part of the universe. We also evaluate a related setting where the universe is built from higher-capacity independently trained image and text encoders on \textsc{Flickr30k}. Across these settings, post-hoc alignment methods reach roughly $0.70$--$0.75$ R@1 and $0.90$--$0.95$ R@5, despite using image and text encoders that were not trained together. As a reference for the same split, CLIP-B/16 reaches $0.863/0.976$ R@1/R@5, as expected for a model trained directly with paired image--text supervision. Future work could study richer post-hoc multimodal alignment methods beyond the alignment families considered here.

\paragraph{Unequal feature dimensions.}
For retrieval experiments with heterogeneous feature dimensions, we use PCA fitted on the training split to map spaces to a common dimension. We compare this preprocessing with zero-padding to the maximum feature dimension on \textsc{Flickr8k}, using a 9-space multimodal universe without CLIP: three audio spaces (HuBERT-B, HuBERT-L, Wav2Vec2-B), three text spaces (BERT, GTR-T5~\cite{ni2022large}, ALBERT), and three image spaces (DINOv2-B, DINOv2-L, ResNet-18).

\begin{table}[h]
\centering
\begin{tabular}{lcc}
\toprule
Dimension matching & GPA & GCPA \\
\midrule
Zero-padding & 0.483 & 0.486 \\
PCA & \textbf{0.537} & \textbf{0.548} \\
\bottomrule
\end{tabular}
\caption{\textbf{Effect of dimensionality matching on Flickr8k multimodal retrieval.} Values are mean R@1 retrieval performance across encoder pairs.}
\label{tab:dimension_ablation}
\end{table}

Table~\ref{tab:dimension_ablation} shows GCPA improves over GPA in each case. We generally observe the PCA-based preprocessing to be beneficial for the alignment baselines in retrieval experiments with heterogeneous spaces. We do not study dimensionality reduction in depth, and leave a broader study of dimension-matching strategies to future work
(Section~\ref{app:limitations}).

\subsection{Model choice and universe composition}
\label{app:model-selection-stability}

We study how the choice of spaces affects the corrected universe. This analysis is distinct from the continual-addition experiment in
Appendix~\ref{app:continual-addition}. There, the question is whether a new space can be integrated efficiently after a universe has already been learned. Here, we ask which candidate spaces improve a fixed universe and which ones can make the correction less reliable.

We run a probing experiment on CIFAR-100 where the base universe is formed from five standard RGB vision encoders: DeiT-Tiny, DeiT-Small, ViT-B/16 at $224$ resolution, ViT-B/16 at $384$ resolution, and BEiT-B. We then add one candidate space at a time and measure its effect on the original base models. The candidate set contains standard RGB DINOv2-S/B/L, ResNet-18/50, and ConvNeXt-B spaces, together with Canny-corrupted variants of DINOv2-B, ResNet-18, and ConvNeXt-B.

For each candidate, we report three quantities. Consensus alignment is the mean cosine similarity between the candidate projection and the base-model consensus in the updated universe. $\Delta$ Acc. is the change in probing accuracy on the original base models after adding the candidate. Candidate Acc. is the mean cross-model probing accuracy between the candidate and the base models after adding the candidate to the universe.

\begin{table}[t]
\centering
\resizebox{\linewidth}{!}{
\begin{tabular}{lrrr}
\toprule
Candidate & Consensus align. & $\Delta$ Acc. & Candidate Acc. (\%) \\
\midrule
DINOv2-S RGB & 0.846 & +0.26 & 76.78 \\
DINOv2-B RGB & 0.807 & +0.99 & 80.54 \\
DINOv2-B Canny & 0.644 & -0.88 & 75.36 \\
DINOv2-L RGB & 0.765 & +1.47 & 82.33 \\
ResNet-18 RGB & 0.770 & -0.21 & 66.34 \\
ResNet-18 Canny & 0.604 & -0.91 & 57.10 \\
ResNet-50 RGB & 0.740 & +1.30 & 70.39 \\
ConvNeXt-B RGB & 0.860 & +0.54 & 74.23 \\
ConvNeXt-B Canny & 0.742 & -0.23 & 69.19 \\
\bottomrule
\end{tabular}
}
\caption{\textbf{Model-choice analysis on CIFAR-100.} Each candidate space is added to the same fixed base universe.}
\label{tab:universe_guidelines}
\vspace{-0.25cm}
\end{table}

Table~\ref{tab:universe_guidelines} shows that model capacity, candidate accuracy, and consensus alignment do not by themselves determine whether a candidate improves the universe. DINOv2-L has the strongest candidate accuracy ($82.33\%$) and gives the largest base-model gain ($+1.47$ points), but ResNet-50 gives a comparable gain ($+1.30$ points) despite lower candidate accuracy ($70.39\%$). ConvNeXt-B has the highest consensus alignment ($0.860$), but gives only a modest base-model gain ($+0.54$ points). Thus, useful candidates are not simply the largest, strongest, or most consensus-aligned models in isolation; the diagnostics must be read together.

The Canny-corrupted candidates have lower consensus alignment than their RGB counterparts. DINOv2-B drops from $0.807$ to $0.644$, ResNet-18 drops from $0.770$ to $0.604$, and ConvNeXt-B drops from $0.860$ to $0.742$. DINOv2-B Canny reduces base accuracy by $0.88$ points, ResNet-18 Canny reduces it by $0.91$ points, and ConvNeXt-B Canny reduces it by $0.23$ points. Their candidate accuracies remain lower than their RGB counterparts. This suggests that risky candidates are those with lower cosine similarity to the existing base model consensus and weak task utility.

The problematic cases are those where lower consensus alignment coincides with weak or negative gains on the base universe. Taken together, the results suggest that model choice should consider agreement with the existing base-model consensus, the effect on the original base models, and the candidate's cross-model utility.

\paragraph{Effect on within-space class structure.}
\label{app:selfrefinement}

We also compare the original spaces with their projections into the learned universe. To measure class separation, we use the mean Fisher ratio, which is higher when class centers are farther apart and samples within each class are closer to their own class center. Averaged across all encoders in this experiment, the mean Fisher ratio is $0.452$ for the original standardized spaces, $0.452$ after projection into the GPA universe, and $0.817$ after projection into the GCPA universe.

These results suggest that the correction can improve class separation after projection into the shared universe. We leave a more systematic study of this self-refinement effect to future work.

\section{Geometry correction (GCPA) algorithm}
\label{app:gcpa_details}

This section specifies the universe-level geometry correction used by \textbf{GCPA}.
We first fit an $M$-way orthogonal universe using GPA, obtaining per-space rotations
$\{\Omega_m\}_{m=1}^M$ and universe embeddings $X_m^\star = X_m \Omega_m \in \mathbb{R}^{N\times d}$.
GCPA then applies a shared correction in universe coordinates that encourages agreement across spaces while keeping the corrected directions close to the GPA universe.

\paragraph{Consensus target.}
Let $u_{m,i}$ denote the $i$th row of $X_m^\star$.
 We form per-space unit directions
$\hat{u}_{m,i} = u_{m,i}/\|u_{m,i}\|_2
$ and define a per-sample consensus direction
\begin{equation}
c_i
=
\mathrm{norm}\!\left(\frac{1}{M}\sum_{m=1}^M \hat{u}_{m,i}\right),
\qquad
C = [c_1;\dots;c_N] \in \mathbb{R}^{N\times d},
\label{eq:consensus_dir}
\end{equation}
where $\mathrm{norm}(\cdot)$ denotes row-wise $\ell_2$ normalization.

\paragraph{Correction module.}
We learn a shared residual map $T_\theta:\mathbb{R}^d\to\mathbb{R}^d$ applied to row-normalized universe embeddings.
It is initialized near the identity and outputs a row-normalized corrected direction,
\begin{equation}
T_\theta(u)=\mathrm{norm}\!\big(\hat{u} + \Delta_\theta(\hat{u})\big),
\qquad \hat{u}=\mathrm{norm}(u),
\label{eq:gcpa_corrector}
\end{equation}
where $\Delta_\theta$ is a small MLP shared across all spaces.

\paragraph{Training objective with trust penalty.}
We train $T_\theta$ to match the consensus direction $c_i$ from universe vectors $u_{m,i}$ sampled across spaces and indices.
For a sampled pair $(u_{m,i},c_i)$, let $y=T_\theta(u_{m,i})$ and $\hat{u}=\mathrm{norm}(u_{m,i})$.
We use the cosine loss
\[
\ell_{\mathrm{align}}(y,c_i)=1-\langle y,c_i\rangle,
\]
and measure deviation from the original GPA direction by
\[
\ell_{\mathrm{drift}}(y,\hat{u})=1-\langle y,\hat{u}\rangle.
\]

The trust penalty activates only when drift exceeds a tolerance $\tau$:
\begin{equation}
\ell_{\mathrm{trust}}(y,\hat{u}) = \max\{0, \ell_{\mathrm{drift}}(y,\hat{u}) - \tau\}.
\end{equation}
The final objective for a minibatch is
\begin{equation}
\mathcal{L}_{\mathrm{GCPA}}
=
\mathbb{E}\big[\ell_{\mathrm{align}}\big]
\;+\;
\lambda\,\mathbb{E}\big[\ell_{\mathrm{trust}}\big],
\label{eq:gcpa_loss}
\end{equation}
with fixed $(\tau,\lambda)$ across all benchmarks.

\paragraph{Optimization.}
We optimize Eq.~\ref{eq:gcpa_loss} on the training split using a standard minibatch procedure.

\paragraph{Using the corrector at inference.}
For a source space $m$, the universe map becomes
\begin{equation}
\mathrm{toU}(x;m)
=
\begin{cases}
x\Omega_m, & \text{GPA} \\
T_\theta(x\Omega_m), & \text{GCPA}.
\end{cases}
\end{equation}
Mapping back into a target space $n$ is $\mathrm{fromU}(u;n)=u\Omega_n^\top$.
Because $T_\theta$ is not restricted to be orthogonal, strict orthogonal cycle consistency need not hold after correction; we therefore analyze the correction separately (Appendix~\ref{app:geom_corr}).

\section{Shared-basis alignment (GCCA) algorithm}
\label{app:gcca_algo}

We provide a compact algorithmic summary of the shared-basis alignment used in our experiments.

\begin{algorithm}[t]
\caption{GCCA for $M$-way shared-basis alignment}
\label{alg:gcca}
\begin{algorithmic}[1]
\REQUIRE Matched representations $\{X_m \in \mathbb{R}^{N \times d_m}\}_{m=1}^M$, target rank $R$
\ENSURE Space-to-shared-basis maps $\{Q_m \in \mathbb{R}^{d_m \times R}\}_{m=1}^M$

\STATE For each space $m$, compute a (possibly truncated) thin SVD:
\[
X_m = U_m \Sigma_m V_m^\top .
\]
\STATE Let $r_m$ denote the retained rank of the thin SVD for space $m$ (so $\Sigma_m \in \mathbb{R}^{r_m\times r_m}$ is invertible).

\STATE Form the block matrix $S \in \mathbb{R}^{(\sum_m r_m)\times(\sum_m r_m)}$ with blocks
\[
S_{mm}=(M-1)I_{r_m}, \qquad S_{mn}=-U_m^\top U_n \;\; (m\neq n).
\]

\STATE Compute the $R$ eigenvectors of $S$ with smallest eigenvalues and stack them as
\[
\Phi^\star = \begin{bmatrix}\Phi_1^\star \\ \vdots \\ \Phi_M^\star \end{bmatrix},
\qquad \Phi_m^\star \in \mathbb{R}^{r_m \times R}.
\]

\STATE Define the shared-basis embeddings for samples in each space:
\[
Y_m := U_m \Phi_m^\star \in \mathbb{R}^{N \times R}.
\]

\STATE Define the feature-side maps into the shared basis:
\[
Q_m := V_m \Sigma_m^{-1} \Phi_m^\star,
\quad \text{so that} \quad
X_m Q_m = Y_m.
\]

\RETURN $\{Q_m\}_{m=1}^M$
\end{algorithmic}
\end{algorithm}

\section{Visualisation of Structure Differences}\label{app:visdiff}
Figure~\ref{fig:umap} provides a qualitative view of how alignment changes class structure. We apply UMAP to embeddings in each aligned space and visualize the resulting class clusters. Consistent with the quantitative results in Table~\ref{tab:massive_clustering}, GCPA produces more clearly separated class-specific clusters, reflecting a controlled deformation of the universe that improves cross-space comparability. In contrast, GPA preserves within-space geometry by construction, while GCCA emphasizes shared directions through a learned projection; both often yield cluster layouts that remain closer to the original structure and can exhibit more overlap between classes.

\section{Limitations and Broader Impact}\label{app:limitations} 
Our framework assumes that the aligned models share a meaningful latent consensus and that at least some reliable cross-space correspondences are available. When these assumptions fail, alignment quality can degrade. This is already reflected in our experiments: although the shared-universe formulation is robust to substantial correspondence corruption, the geometry-correction step can become unstable when the universe is supported by too few or too weak anchor models, or when the induced consensus is itself distorted. In such cases, the correction may reinforce mismatch rather than remove nuisance variation. More broadly, GCPA is best suited to settings where multi-model agreement reflects shared semantics, and is less appropriate when the participating spaces are highly non-isomorphic or encode incompatible task structure. Our retrieval experiments also use PCA to place heterogeneous feature spaces into a common dimension before alignment; this preprocessing may discard information, and alternative strategies for handling unequal dimensionality deserve further study.

A second limitation is that the shared-universe formulation improves scalability but does not remove computational tradeoffs. The number of learned maps grows linearly rather than quadratically, and new models can be added by fitting a single additional map into an existing universe, but both GPA fitting and GCPA correction still incur nontrivial cost as the number of models, dimensionality, and batch size increase. In addition, although GCPA is anchored to the GPA scaffold, it does not retain the strict isometric guarantees of purely orthogonal alignment. This tradeoff improves retrieval and clustering, but downstream uses requiring exact preservation of within-model geometry may still prefer GPA.

From a broader-impact perspective, multi-model alignment can have both positive and negative consequences. It may enable interoperability across independently trained systems and reduce the need for task-specific pairwise engineering. However, a shared universe may also transmit undesirable properties across models, including bias, spurious correlations, or other harmful structure already present in one or more constituent spaces. In sensitive domains, alignment may also increase privacy or security risks by making model internals more comparable across systems, potentially easing cross-model inference or representation leakage. We do not study these issues here, and view fair, privacy-preserving, and access-controlled alignment as important directions for future work.

\begin{figure*}
    \centering
    \includegraphics[width=1\linewidth]{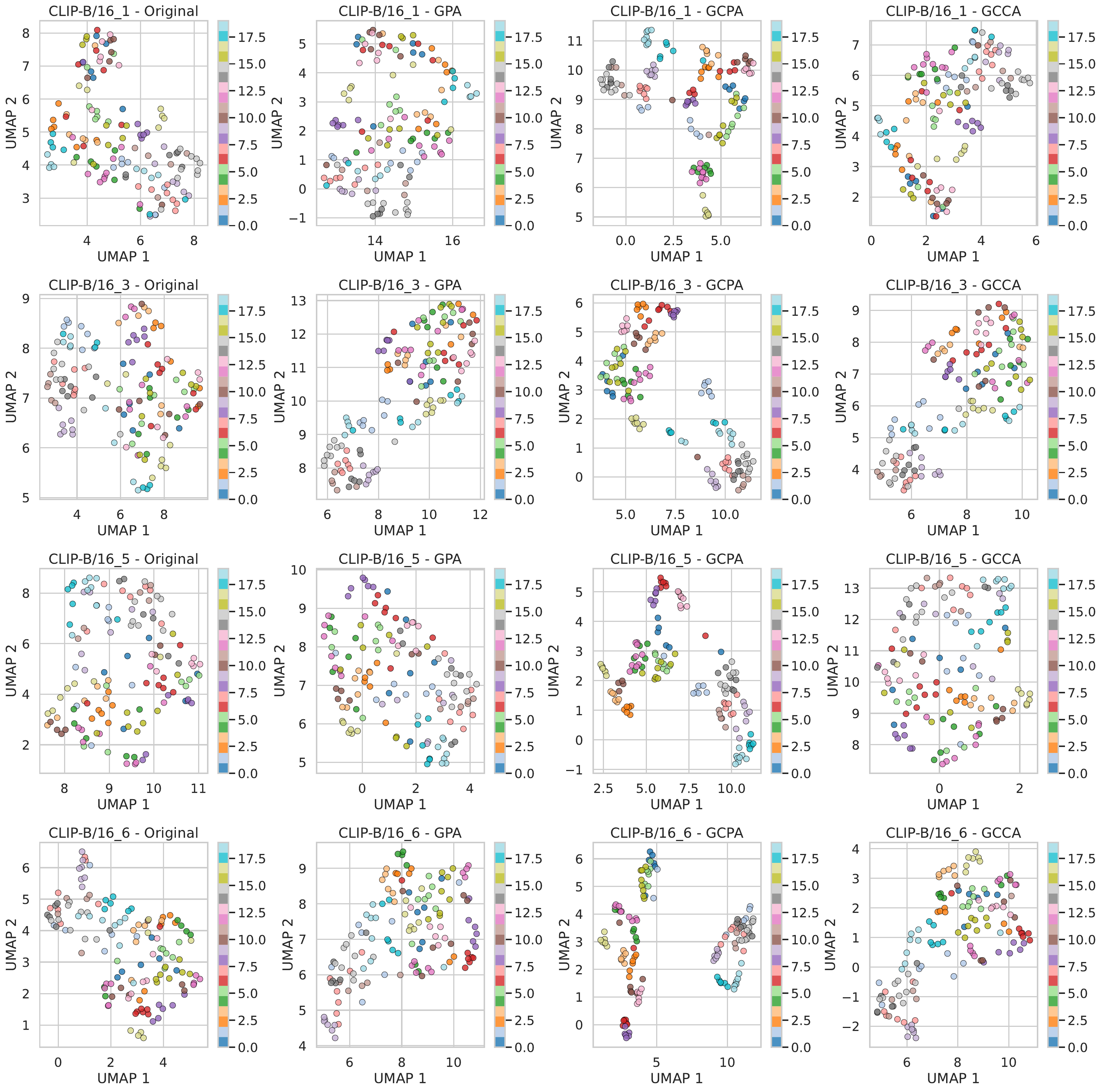}
    \caption{UMAP representation of images from 20 representative classes from the \textsc{Market-1501} dataset. The first column computes distances between images directly in the original model space, while the others are the universal spaces made using GPA, GCPA, and GCCA. Each line of subplots is made using input from a different camera. Dots with the same colors represent images from the same classes. }
    \label{fig:umap}
\end{figure*}

\end{document}